\documentclass[twoside,11pt]{article}

\usepackage[preprint]{jmlr2e}

\usepackage{hyperref}

\usepackage[utf8]{inputenc}
\usepackage[T1]{fontenc}

\usepackage{microtype}
\usepackage{breakcites}

\newcommand{\ie}{{i.e.,\ }}
\newcommand{\eg}{{e.g.,\ }}

\newcommand{\dataset}[1]{\mbox{#1}}

\usepackage{booktabs}

\usepackage{graphicx}
\usepackage{subcaption}
\graphicspath{{plots/}}

\usepackage{pgfplots}
\DeclareUnicodeCharacter{2212}{−}
\usepgfplotslibrary{groupplots,dateplot}
\usetikzlibrary{patterns,shapes.arrows}
\pgfplotsset{
  compat=newest,
  every axis plot/.append style={
    line width=1,
  },
  xtick style={color=black},
  ytick style={color=black},
  tick label style={font=\footnotesize},
  tick align=outside,
  tick pos=left,
  major grid style={opacity=.7},
  minor grid style={opacity=.4},
  label style={font=\small},
  legend style={
    font=\small,
    fill opacity=0.8,
    draw opacity=1,
    text opacity=1,
    draw=white!50!black
  },
  legend cell align=left,
}

\usepackage{amsmath,amsfonts,mathtools}
\usepackage[mathscr]{euscript}

\def\R{{\mathbb{R}}}

\def\X{{\mathcal{X}}}
\def\Y{{\mathcal{Y}}}

\def\regret{{\mathrm{Reg}}}

\DeclareMathOperator*{\argmax}{arg\,max}

\DeclareMathOperator*{\softmax}{softmax}

\newcommand{\ens}[1]{\left\{#1 \right\}}
\newcommand{\enscond}[2]{\ens{#1 \,:\, #2}}

\def\SymDiff{{\bigtriangleup}}

\newcommand{\powerset}[1]{\mathcal{P}(#1)}

\newcommand{\preimage}[2][\cdot]{#2^{-1} [ #1 ]}

\newcommand{\interior}[1]{\mathrm{Int} \, #1}
\newcommand{\support}[1]{\mathrm{Supp} ( #1 )}

\def\Pr{{\mathbb{P}}}
\newcommand{\pr}[2][]{\Pr_{#1} \left[ #2 \right]}
\newcommand{\prcond}[2]{\pr{ #1 \mid #2 }}
\newcommand{\E}[2][]{\mathbb{E}_{#1} \left[ #2 \right]}
\newcommand{\Econd}[3][]{\E[#1]{ #2 \mid #3 }}

\newcommand{\ProbSimplex}[1]{\bigtriangleup_{#1}}

\newcommand{\eqdef}{\vcentcolon=}

\DeclareSymbolFont{bbold}{U}{bbold}{m}{n}
\DeclareSymbolFontAlphabet{\mathbbold}{bbold}
\newcommand{\ind}[1]{\mathbbold{1}_{ #1 }}

\DeclareMathOperator{\KLdiv}{D_\mathrm{KL}}

\def\cD{{\mathcal{D}}}

\def\cK{{\mathcal{K}}}

\makeatletter
\renewcommand*{\@opargbegintheorem}[3]{\trivlist
      \item[\hskip \labelsep{\bfseries #1\ #2}] \textbf{(#3)}\ \itshape}
\makeatother

\usepackage{cleveref}
\newcommand{\Crefparenthesis}[1]{%
  \Crefformat{equation}{Equation~##2##1##3}%
  \Crefmultiformat{equation}{Equations~##2##1##3}{ and~##2##1##3}{, ##2##1##3}{ and~##2##1##3}%
  \Cref{#1}%
  \Crefformat{equation}{Equation~(##2##1##3)}%
  \Crefmultiformat{equation}{Equations~(##2##1##3)}{ and~(##2##1##3)}{, (##2##1##3)}{ and~(##2##1##3)}%
}

\def\Risk{{\mathcal{R}}}
\def\Error{{\mathcal{E}}}
\def\Info{{\mathcal{I}}}

\def\RegretFixed{{\regret_{\mathrm{top}}}}
\def\RegretAdapt{{\regret_{\mathrm{avg}}}}

\def\TopKSet{{S}}
\def\AverageKSet{{\mathscr{S}}}

\usepackage{algorithm}
\usepackage{algorithmicx}
\usepackage{algpseudocode}
\usepackage[frozencache=true,cachedir=.]{minted}

\usepackage{todonotes}
\usepackage{letltxmacro}
\LetLtxMacro{\oldtodo}{\todo}
\renewcommand{\todo}[2][]{\tikzexternaldisable\oldtodo[#1]{#2}\tikzexternalenable}

\jmlrheading{?}{2021}{?-?}{12/21}{?}{?}{Titouan Lorieul, Alexis Joly and Dennis Shasha}

\ShortHeadings{Classification Under Ambiguity}{Lorieul, Joly and Shasha}
\firstpageno{1}

\title{Classification Under Ambiguity: \\ When Is Average-$K$ Better Than Top-$K$?} 

\author{%
  \name Titouan Lorieul \email titouan.lorieul@inria.fr \\
  \name Alexis Joly \email alexis.joly@inria.fr \\
  \addr Zenith \\
  Inria, LIRMM, Université de Montpellier \\
  161 rue Ada, 34095 Montpellier, France
  \AND
  \name Dennis E. Shasha \email shasha@cims.nyu.edu \\
  \addr Courant Institute \\
  New York University \\
  New York, USA
}

\editor{}

\begin{document}

\maketitle

\begin{abstract}%
  When many labels are possible, choosing a single one can lead to low precision.
  A common alternative, referred to as \emph{top-$K$ classification}, is to choose some number $K$ (commonly around 5) and to return the $K$ labels with the highest scores.
  Unfortunately, for unambiguous cases, $K>1$ is too many and, for very ambiguous cases, $K \leq 5$ (for example) can be too small.
  An alternative sensible strategy is to use an adaptive approach in which the number of labels returned varies as a function of the computed ambiguity, but must average to some particular $K$ over all the samples.
  We denote this alternative \emph{average-$K$ classification}.
  This paper formally characterizes the ambiguity profile when average-$K$ classification can achieve a lower error rate than a fixed top-$K$ classification.
  Moreover, it provides natural estimation procedures for both the fixed-size and the adaptive classifier and proves their consistency.
  Finally, it reports experiments on real-world image data sets revealing the benefit of average-$K$ classification over top-$K$ in practice.
  Overall, when the ambiguity is known precisely, average-$K$ is never worse than top-$K$, and, in our experiments, when it is estimated, this also holds.
\end{abstract}


\section{Introduction}

Consider the problem of assigning a label to an object.
In applications where the label noise is low, predicting a single label works well.
However, in a wide range of applications, at least some data items may be ambiguous due to noise or occlusion. In such cases, even expert human annotators may disagree on what the true label should be
and may prefer to give a list of possible answers, while filtering out the classes that certainly are wrong.

\begin{figure}
  \centering
  
  \begin{subfigure}{.9\textwidth}
    \centering
    \includegraphics[width=\textwidth]{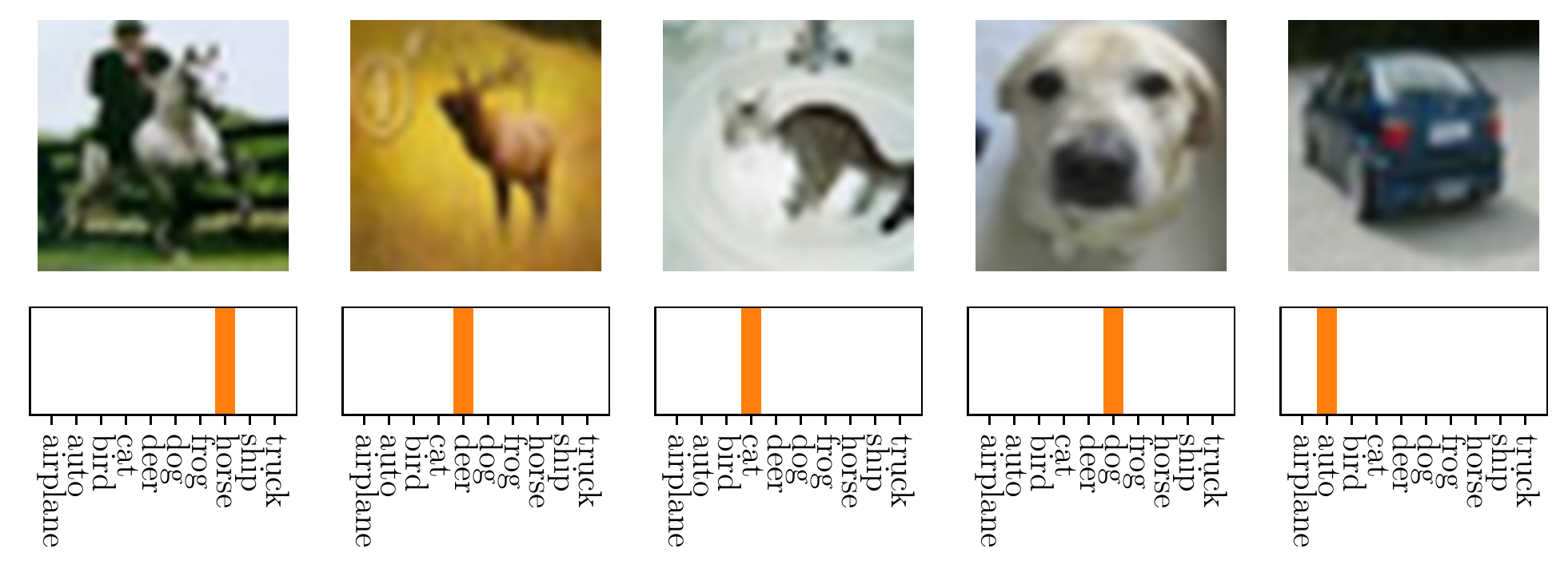}
    \caption{
      Non-ambiguous examples.
    }
  \end{subfigure}
  \\
  \vspace{.25cm}
  \begin{subfigure}{.9\textwidth}
    \centering
    \includegraphics[width=\textwidth]{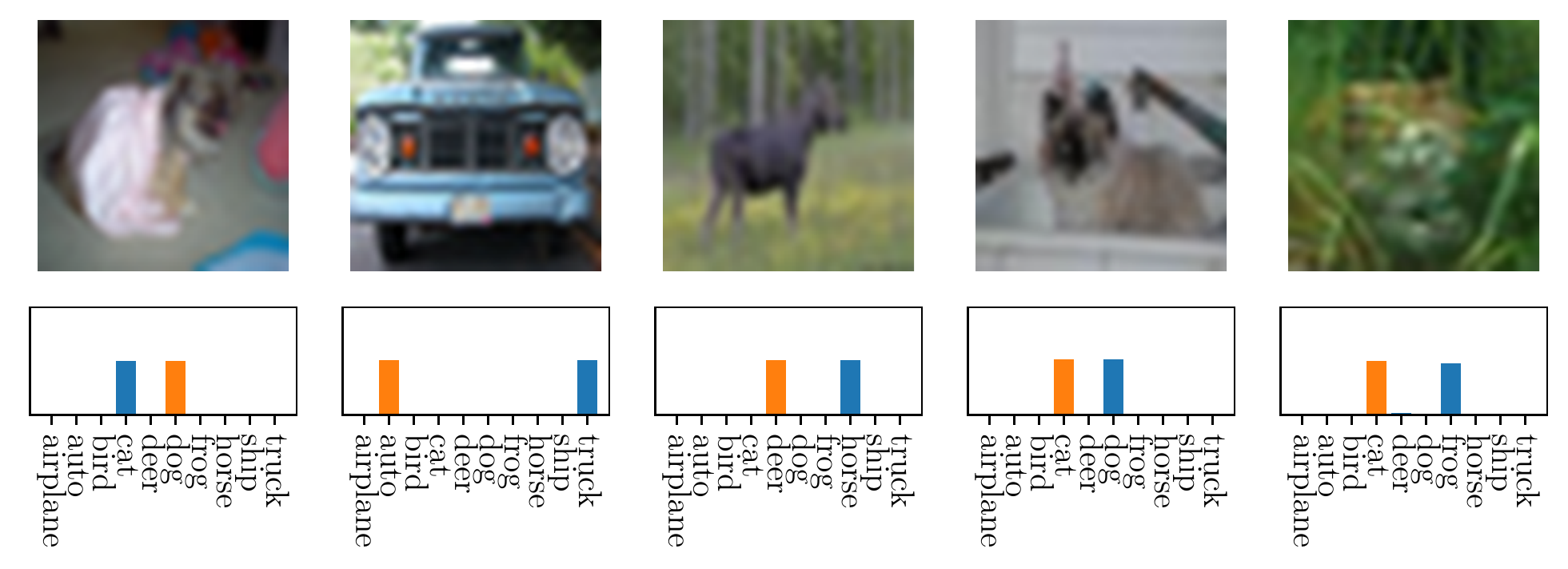}
    \caption{
      Ambiguity between two classes.
    }
  \end{subfigure}
  \\
  \vspace{.25cm}
  \begin{subfigure}{.9\textwidth}
    \centering
    \includegraphics[width=\textwidth]{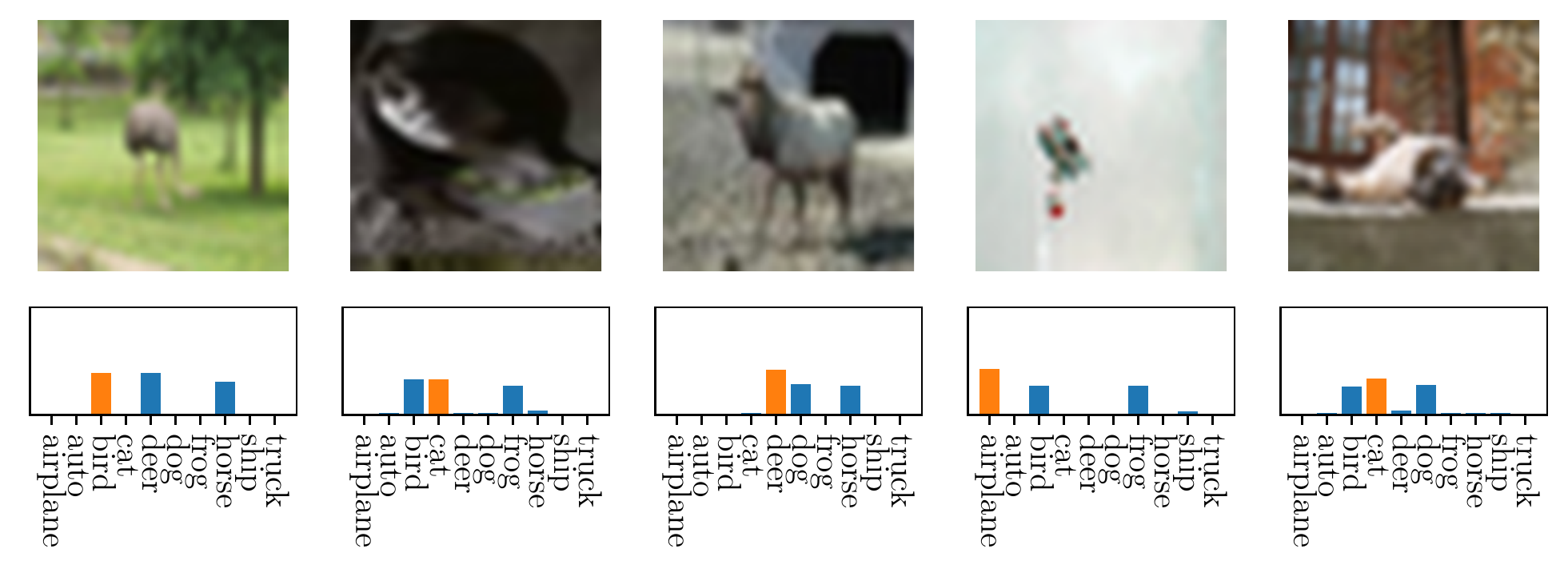}
    \caption{
      Ambiguity among three classes.
    }
  \end{subfigure}
  
  \caption{
    Various degrees of ambiguity according to human annotators on \dataset{CIFAR-10} collected by \citet{Peterson2019}.
    The orange class corresponds to the most probable class according to these annotators.
  }
  \label{fig:cifar10h-human-ambiguity-examples}
\end{figure}

Researchers have found that even on simple tasks, such as \dataset{CIFAR-10}, humans sometimes disagree \citep{Peterson2019}.
This is illustrated in \Cref{fig:cifar10h-human-ambiguity-examples} which shows images with different levels of disagreement according to annotators.
The first row displays unambiguous images, whereas the second and third rows contain images with, respectively, two and three possible classes.
In all cases, a single object is present in the image, however, for some images, the low image resolution does not allow a person to determine precisely what this object is.\footnote{\dataset{CIFAR-10} images have a size of 28x28.}
However,  filtering (ruling out some classes) is still  useful  when there is ambiguity.

\begin{figure}[t]
  \centering
  \input{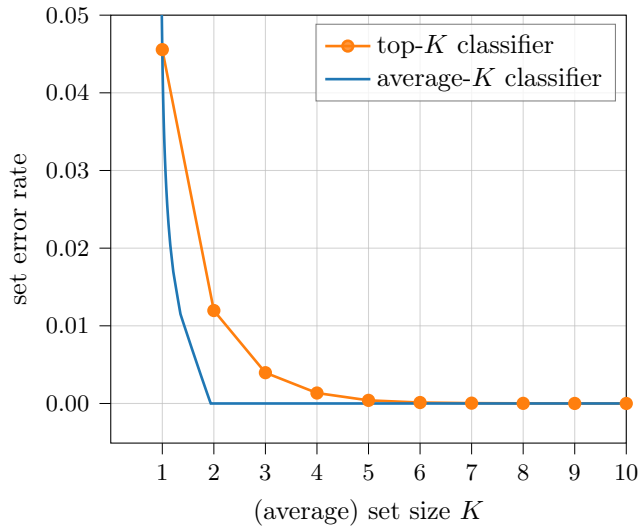}
  \caption{
    Comparison of error rate of the top-$K$ strategy with the average-$K$ strategy studied in this paper computed on the human uncertainty annotations of \dataset{CIFAR-10H} \citep{Peterson2019}.
    Lower is better.
  }
  \label{fig:cifar10h-set-error-rates-example}
\end{figure}

In this paper, we study an adaptive classification approach where computational  classifiers are allowed to act like human experts by responding with a list of candidate labels.
The size of the list may vary depending on the sample being classified but is constrained to have a mean value of $K$.
The value $K$ may be related to the resources available to validate or exclude each candidate classification (\eg in a medical setting where validating each possible diagnosis may require one or more specialized tests).
We denote this as \emph{average-$K$ classification}, a generalization of \emph{top-$K$ classification}.

In our problem setting, as in \Cref{fig:cifar10h-human-ambiguity-examples}, there is a single correct label, but the input provided may be ambiguous.
This is a different task from multi-label classification \citep{Zhang2013} where there are multiple true answers, all of them correct, which are given during training.
\Cref{fig:cifar10h-set-error-rates-example} shows how, on \dataset{CIFAR-10H}, an adaptive average-$K$ strategy has the potential to lower the error rate compared to always predicting the same number of classes for each sample.

The paper is organized as follows.
After formalizing the average-$K$ classification task, we determine the contexts in which the adaptive strategy is most beneficial compared to top-$K$ classification.
Moreover, we give estimation procedures to estimate those strategies and show that they are consistent.
Finally, we perform experiments on real-world image data sets to determine whether average-$K$ classification can be helpful in practice.

\section{Related Work}
\label{sec:related-work}

Uncertainty is naturally present in classification tasks \citep{Kendall2017,Huellermeier2021}.
There exist different approaches in the literature to handle it.
In this paper, we focus on \emph{decision-theoretic} methods.
They consist in directly learning a decision rule based on this uncertainty rather than attempting to quantify it explicitly.
There are two main classification settings which are related to our case: classification with reject option and set-valued classification.

\emph{Classification with reject option} consists in allowing a predictor to refuse to answer to a query when in presence of uncertainty.
It has been studied for a long time, one of the earliest works being the one of \citet{Chow1957}, and since then a growing literature has been developed with works ranging from theoretical formulation, statistical analysis to practical algorithm \citep[see among others][]{Chow1970,Herbei2006,Bartlett2008,Cortes2016}.
These works focus on the binary case offline setting but this approach has been extended to other settings such as the multi-class case \citep{Ramaswamy2018}, online learning \citep{Kocak2019}, etc.
We refer the reader to the survey of \citet{Hendrickx2021} for a more complete overview of literature of this field.

On the other hand, work on the \emph{set-valued prediction} problem is more recent.
It considers the setting where there is a single correct label, and the goal is to output a small set that contains this single correct label.
Some works, such as the ones from \citet{Lapin2016,Berrada2018,Blondel2020,Yang2020}, tackle the \emph{top-$K$ prediction} problem where one is interested in predicting a fixed set size ($K$) for all samples.
$K$ can be any integer value, typically small compared to the number of classes.
These works focus on how to more efficiently learn such rules.

Another line of work considered adaptive versions of this setting.
It has been studied under several names corresponding to distinct objectives.
These different formulations come from the fact that a trade-off needs to be found between the \emph{accuracy} and the \emph{informativeness} of the set-valued predictor.
The  way this trade-off is characterized determines the different formulations.
One of the earliest attempts was branded \emph{class-selection} or \emph{class-selective rejection} by \citet{Ha1997}.
It consists in formulating this trade-off as a linear combination which can equivalently be seen as an arithmetic mean between the error rate and the average predicted set size.
\citet{Coz2009} express this trade-off as an F-measure, \ie a harmonic averaging, and denote it as \emph{nondeterministic classifiers}.

Another important formulation of the problem is \emph{conformal prediction} \citep{Vovk2005,Shafer2008} where the aim is to guarantee that the error rate is no bigger than a certain value---set predictor \emph{validity}---while keeping the set size as small as possible---set predictor \emph{efficiency}.
It has been successfully applied to real-world image datasets by \citet{Angelopoulos2021}.
\citet{Sadinle2019} explored a special instance of this problem, denoted \emph{least ambiguous set-valued classification}, in which the aim is to minimize the average size of the predicted sets while satisfying a constraint on the average error rate.
By contrast, closer to our work, \citet{Denis2017} define \emph{confidence sets} which aim to set a constraint on the average set size and to minimize the error rate for that value.
Those works and subsequent ones built on top of them define new formulations of the problem \citep{Mortier2021}, provide algorithms \citep{Wang2018}, statistical analyses \citep{Chzhen2021}, etc.

In summary, many formulations have been proposed in the literature which result in many different optimal strategies.
However, to the best of our knowledge, no attempt has been made at characterizing the conditions underlying the usefulness of formulating the classification task as an adaptive set-valued classification problem.
For instance, \citet{Ha1997} computes upper-bounds on error rate and average set size but, as stated in the paper, these bounds are distribution-free.
They do not study conditions on the distribution in which adaptation is useful.
Our goal is to understand and to characterize problems for which adaptation should be used.

\section{From Top-$K$ to Average-$K$ Classification}
\label{sec:preliminaries}

In this section, we detail the formalization and notations that will be used in the rest of the paper.
We first briefly recall the multi-class classification and top-$K$ classification settings.
Then, we extend them to the average-$K$ classification task by revisiting a formalism that was previously presented by \citet{Denis2017} to make it more broadly applicable.
Finally, we synthesize the main contributions of the paper.

\subsection{Multi-class Classification}
In the multi-class classification setting \citep{Tewari2007}, we are given a training data set $(x_i,y_i)_{i \in \ens{1,\dots,n}}$ where, for all $i$, the input $x_i$ belongs to an input space $\X \subset \R^d$, \eg the space of images, and the output $y_i$ to an output space $\Y = \ens{1,2,\dots,C}$, the set of class labels, where $C$ is the number of classes.
The joint space $\X \times \Y$ is a probabilistic space and the data points are sampled from the joint probability measure $\Pr_{X,Y}$.
This joint probability can, in turn, be decomposed into the marginal probability measure on $\X$, denoted $\Pr_X$, and the conditional probability of $Y$ given $X$, denoted $\eta(x)$:
\begin{equation*}
  \eta_k(x) \eqdef \prcond{ Y=k }{ X=x }.
\end{equation*}

In the classical setting, the aim is to find a classifier $h: \X \to \Y$ which produces a single prediction for each input and generalizes well on unseen data.
This is formalized by trying to minimize the risk defined by the 0-1 error rate,
\begin{equation}
  \Risk(h) \eqdef \E[X,Y]{ \ind{Y \neq h(X)} } = \pr[X,Y]{ Y \neq h(X) }
  \label{eq:risk-top-1-classification}
\end{equation}
where $\ind{P}$ is equal to $1$ if the logical proposition $P$ is true and to $0$ otherwise.

Any predictor satisfying, for all $x \in \X$,
\begin{equation*}
  h^*(x) \in \argmax_{k \in \Y} \eta_k(x)
\end{equation*}
minimizes the previous risk, \ie $\Risk(h^*) = \inf_h \Risk(h)$.
Such a predictor is called a Bayes predictor.
For simplicity, we usually assume that there is a single minimizer and write $h^*(x) = \argmax_k \eta_k(x)$.
In this case, we have
\begin{equation*}
   \Risk(h^*) = 1 - \E[X]{ \max_k \eta_k(X) }.
\end{equation*}

In general, the minimizer does not achieve a risk of zero.
A risk of zero occurs only when $\max_k \eta_k(x) = 1$ almost everywhere. 
However, in many practical data sets, such as \dataset{CIFAR-10} \citep{Krizhevsky2009}, \dataset{ImageNet} \citep{Russakovsky2015}, \dataset{iNaturalist2018} \citep{VanHorn2018} and others, there is some intrinsic ambiguity in the task and we have $\Risk(h^*) > 0$.
From a top-1 classification perspective, this risk is irreducible: it is the lowest achievable risk.
The risk can, however, be reduced by  predicting sets.

\subsection{Top-$K$ Classification}
\label{sec:top-k-formulation}

The most natural and direct extension to top-1 classification is to take the top-$K$ most probable classes and predict all $K$.
The learning problem can be formulated as follows.
Denoting $\powerset{\Y}$ the power set of $\Y$, \ie the set of all subsets of $Y$, the classifier $\TopKSet$ is now of the form $\TopKSet: \X \to \powerset{\Y}$: it takes values from $\X$ and predicts a set of labels from $\Y$.
Moreover, it has the following additional constraint:
\begin{equation*}
  \forall x \in \X, \quad |\TopKSet(x)| = K.
\end{equation*}
The objective, the risk of \Cref{eq:risk-top-1-classification}, is then modified to minimize the set error rate:
\begin{equation*}
  \Error(\TopKSet) \eqdef \E[X,Y]{ \ind{Y \notin \TopKSet(X)} } = \pr[X,Y]{ Y \notin \TopKSet(X)}.
\end{equation*}
This expression yields the top-$K$ error rate.

As expected, this error rate is minimized by predicting the top-$K$ most probable classes for any input $x$ \citep{Lapin2016}.
We denote by $\eta_{\sigma_x(k)}(x)$ the re-ranking of $\eta_k(x)$ in decreasing order of probability, \ie $\eta_{\sigma_x(1)}(x) \geq \eta_{\sigma_x(2)}(x) \geq \dots \geq \eta_{\sigma_x(C)}(x)$.
The minimizer of the top-$K$ error rate can then be written as
\begin{equation}
  \label{eq:optimal-top-k-classifier}
  \forall x \in \X, \quad \TopKSet_K^*(x) \eqdef \enscond{ \sigma_x(k) }{ k \in \ens{1,\dots,K} } .
\end{equation}
To simplify  notation, in the rest of the document, we will often use the following notation for the re-ordering of $\eta_k(x)$:
\begin{equation*}
  \tilde{\eta}_k(x) := \eta_{\sigma_x(k)}(x).
\end{equation*}

This top-$K$ classification approach is the most direct generalization of top-1 classification, but there is no reason to  predict the same number of classes for all the samples.
We thus propose to consider adaptive set-valued classifiers which relax this constraint.

\subsection{Average-$K$ Classification}
\label{sec:average-K-formulation}

Consider a predictor $\AverageKSet$ (pronounced ``Euler script S'') which outputs any subset of $\Y$.
It is also of the form
\begin{equation*}
  \AverageKSet: \X \to \powerset{\Y},
\end{equation*}
where $\powerset{\Y}$ is the power set of $\Y$.
The naive goal is, as previously, to maximize the chance of predicting the good label.
(Note that, in general, the predictor $\AverageKSet$ may output the empty set $\emptyset$.)

However,  considering only the set error rate, \ie
\begin{equation*}
  \Error(\AverageKSet) \eqdef \pr[X,Y]{ Y \notin \AverageKSet(X) },
\end{equation*}
is too naive, because the predictor $\AverageKSet_{\mathrm{all}}(x) = \Y$ that always predict all  classes is very accurate, but completely uninformative.
Thus, this accuracy needs to be balanced with a measure of informativeness $\Info(\AverageKSet)$.
There are several ways to define  accuracy and informativeness and to measure the trade-off between them.
They give different formulations of the set-valued classification problem which are adapted to different scenarios, as we discussed in
 \Cref{sec:related-work}. 

Because we are interested in a generalization of top-$K$, we propose to relax the constraint of a fixed set size by changing it to a constraint on the \emph{average} set size which we take as measure of informativeness:
\begin{equation*}
  \Info(\AverageKSet) \eqdef \E[X]{ |\AverageKSet(X)| }.
\end{equation*}
This gives the following optimization problem
\begin{equation}
  \label{eq:constrained-optimization-formulation}
  \begin{split}
    \min_{\AverageKSet} \; \pr[X,Y]{ Y \notin \AverageKSet(X)} \\
    \text{s.t.} \; \E[X]{ |\AverageKSet(X)| } \leq \cK
  \end{split}
\end{equation}
where the constraint on the average set size, $\cK > 0$, is not limited to integers but can take any real value.
This setting is similar to the setting studied by \citet{Denis2017}.
However, here, we do not make the simplifying continuity assumption which enforces the absence of ties.
This leads to some changes in the expression of the solution of \Cref{eq:constrained-optimization-formulation}.

This constrained optimization problem has the same solution as the following risk minimization problem proposed by \citet{Ha1997} for a given well-chosen parameter $\lambda$:
\begin{equation*}  
  \min_{\AverageKSet} \; \Risk_\lambda(\AverageKSet)
  \eqdef \Error(\AverageKSet) + \lambda \, \Info(\AverageKSet) \phantom{:}= \E[X,Y]{ \ind{Y \notin \AverageKSet(X)} + \lambda \, |\AverageKSet(X)| } .
\end{equation*}
This newly introduced parameter $\lambda \in \R^+$ controls the balance between both terms.
It is related to $\cK$ in a way that we will explain later in this subsection.
Note that this risk formulation assumes that the cost of predicting additional labels increases proportionally to some cost coefficient $\lambda > 0$.

Given that
\begin{align*}
  \Risk_\lambda(\AverageKSet)
  &= \E[X]{ \left( 1 - \sum_{k \in \AverageKSet(X)} \eta_k(x) \right) + \sum_{k \in \AverageKSet(X)} \lambda } \\
  &= \E[X]{ 1 + \sum_{k \in \AverageKSet(X)} \Big( \lambda - \eta_k(x) \Big) },
\end{align*}
the optimal Bayes predictor $\AverageKSet_\lambda^*$ of this risk can be easily derived and is equal to
\begin{equation*}
  \AverageKSet_\lambda^*(x) \eqdef \enscond{ k \in \Y }{ \eta_k(x) > \lambda }.
\end{equation*}

The link between the cost $\lambda$ and average set size $\cK$ is given by the following function
\begin{equation}
  \label{eq:G-function}
  G_\eta(\lambda) \eqdef \sum_k \pr[X]{ \eta_k(X) > \lambda }.
\end{equation}
Here, $G_\eta(\lambda)$ is exactly the average set size obtained when using the threshold $\lambda$, \ie
\begin{equation*}
  G_\eta(\lambda) = \E[X]{| \AverageKSet_\lambda^*(X) |}.
\end{equation*}
In order to compute the threshold corresponding to a given average set size, we must effectively find the inverse of the previous function.
In general, however, it is not invertible.
Nevertheless, as it is weakly decreasing and right-continuous, we can consider the following generalized inverse function \citep{Embrechts2013} defined as\footnote{Usually, the generalized inverse function is defined using an infimum but, here, as $G_\eta$ is weakly decreasing and right-continuous, the infimum is actually included in the set and is thus its minimum.}
\begin{equation}
  \label{eq:G-inverse-function}
  G_\eta^{-1}(\cK) \eqdef \min \enscond{ \lambda \in [0,1] }{ G_\eta(\lambda) \leq \cK }.
\end{equation}
Using these definitions, as we will show in \Cref{thm:adaptive-set-prediction-excess-risk}, under some conditions easy to match in practice (as elucidated below), the solution of our optimization problem defined in \Cref{eq:constrained-optimization-formulation} can be expressed as
\begin{equation}
  \AverageKSet_\cK^*(x) = \AverageKSet_\cK^+(x) \cup \widetilde{\AverageKSet}_\cK^{=}(x),
  \label{eq:optimal-adaptive-top-k-classifier}
\end{equation}
where
\begin{equation*}
  \AverageKSet_\cK^+(x) = \enscond{ k \in \Y }{ \eta_k(x) > G_\eta^{-1}(\cK) } ,
\end{equation*}
 and $\widetilde{\AverageKSet}_\cK^{=}$ is any \emph{deterministic} classifier predicting a subset of the labels produced by $\AverageKSet_\cK^{=}$ defined as
\begin{equation*}
  \AverageKSet_\cK^{=}(x) = \enscond{ k \in \Y }{ \eta_k(x) = G_\eta^{-1}(\cK) } ,
\end{equation*}
such that $\Info(\widetilde{\AverageKSet}_\cK^{=}) = \cK - \Info(\AverageKSet_\cK^+)$.
Although such a \emph{deterministic} classifier $\widetilde{\AverageKSet}_\cK^{=}$ is not always guaranteed to exist, in practice, however, it  often does.
For instance, the following proposition gives a sufficient condition for its existence.
If it were not to exist, a workaround would consist in replacing it by a stochastic version but, to simplify the results and discussions of this paper, we stick with the deterministic version.

\begin{proposition}[Sufficient condition for the existence of a deterministic average\mbox{-}$\cK$ classifier]
  \label{thm:existence-deterministic-classifier}
  If $\Pr_X$ is continuous (\ie has no atoms), then, for any real $\cK \in [0,C]$, a deterministic classifier $\AverageKSet_\cK^{*}$ can be found.
\end{proposition}

\begin{proof}
  Denote the threshold $\lambda = G_\eta^{-1}(\cK)$ and the missing set size $\rho = \cK - \Info(\AverageKSet_\cK^+)$.
  We first partition the subset of $\X$ in which the equality case occurs into subsets $\X_i$ depending on the number of ties $i \in \ens{1,\dots,C} $ in those subsets:
  \begin{equation*}
    \forall i \in \ens{1,\dots,C}, \quad \X_i = \enscond{ x \in \X }{ | \enscond{ k \in \Y }{ \eta_k(x) = \lambda } | = i }.
  \end{equation*}
  Each of those subsets has a measure which we denote $\mu_i = \pr[X]{\X_i}$.
  Let ${\AverageKSet_\cK^=}_{| \X_i}$ be the restriction of $\AverageKSet_\cK^=$ on $\X_i$, then its average set size is equal to $\Info({\AverageKSet_\cK^=}_{| \X_i}) = i \, \mu_i$.

  We denote $i_0$ the maximum index of the subsets $\X_i$ which we will take entirely, defined by
  \begin{equation*}
    i_0 = \max \enscond{ i' \in \ens{1,\dots,C} }{ \sum_{i \leq i'} \Info({\AverageKSet_\cK^=}_{| \X_i}) = \sum_{i \leq i'} i \, \mu_i \leq \rho } .
  \end{equation*}
  By definition of $i_0$, the measure of subset $\X_{i_0+1}$ satisfies
  \begin{equation*}
    \sum_{i \leq i_0} i \, \mu_i \leq \rho < \sum_{i \leq i_0+1} i \, \mu_i
    \quad \Leftrightarrow \quad
    0 \leq \frac{ \rho - \sum_{i \leq i_0} i \, \mu_i }{ i_0+1 } < \pr[X]{\X_{i_0+1}} .
  \end{equation*}
  As $\Pr_X$ is continuous, we can use a theorem by \citet{Sierpinski1922} \citep[see also][Proposition A.1.8]{Hytonen2016} to extract a subset $\tilde{\X}_{i_0+1}$ of $\X_{i_0+1}$ which measure is equal to
  \begin{equation*}
    \pr[X]{\tilde{\X}_{i_0+1}} = \frac{ \rho - \sum_{i \leq i_0} i \, \mu_i }{ i_0+1 } .
  \end{equation*}
  We can then build our deterministic predictor $\widetilde{\AverageKSet}_\cK^{=}$ as
  \begin{equation*}
    \widetilde{\AverageKSet}_\cK^{=}(x) = \begin{cases}
      \AverageKSet_\cK^{=}(x) & \text{if } x \in ( \cup_{i \leq i_0} \X_i ) \cup \tilde{\X}_{i_0+1} , \\
      \emptyset & \text{otherwise} ,
    \end{cases}
  \end{equation*}
  which satisfies
  \begin{equation*}
    \Info(\widetilde{\AverageKSet}_\cK^{=})
    = \sum_{i \leq i_0} \Info({\AverageKSet_\cK^=}_{| \X_i}) + \Info({\AverageKSet_\cK^=}_{| \tilde{\X}_{i_0+1}})
    = \sum_{i \leq i_0} i \, \mu_i + (i_0+1) \frac{ \rho - \sum_{i \leq i_0} i \, \mu_i }{ i_0+1 }
    = \rho .
  \end{equation*}
  This concludes the proof.
\end{proof}

Note that, here, the continuity assumption is made on the marginal distribution $\Pr_X$ and not on the distribution of the conditional probabilities $\eta_k(X)$ as performed by \citet{Denis2017} which implicitly enforces the absence of ties.
It is moreover weaker in the sense that their assumption actually implies the sufficient condition stated above.
Similarly, if $\Pr_X$ admits a probability density function, then it is continuous but the converse is not necessarily true.
It is thus a broad condition which essentially is not verified when $\Pr_X$ is a discrete distribution or is a mixture of discrete and continuous distributions.

In the absence of ties at $G_\eta^{-1}(\cK)$, \ie when $\forall k, \pr[X]{ \eta_k(x) = G_\eta^{-1}(\cK) } = 0$, the optimal strategy that minimizes the error rate for a given constraint on the average set size is equal to $\AverageKSet_\cK^*(x) = \AverageKSet_\cK^+(x)$, \ie it consists of a simple thresholding on the values of the conditional probabilities $\pr{ Y=k \mid X=x }$.
Indeed, in this case, we have $\Info( \AverageKSet_\cK^+ ) = G_\eta( G_\eta^{-1}( \cK ) ) = \cK$.

When ties are present at $G_\eta^{-1}(\cK)$, then we have $\Info( \AverageKSet_\cK^+ ) = G_\eta( G_\eta^{-1}( \cK ) ) < \cK$, and we need to arbitrarily break the ties to recover the missing $\cK - \Info( \AverageKSet_\cK^+ )$ labels.
This is the role of $\widetilde{\AverageKSet}_\cK^{=}(x)$ which guarantees that, in the end, we have $\Info(\AverageKSet_\cK^*) = \cK$.

Note that this definition of the optimal predictor has been changed compared to \citet{Denis2017}, we have generalized it to account for the possible presence of ties: a second term $\widetilde{\AverageKSet}_\cK^{=}(x)$ was added and the inequality was changed to strict inequality in the first term $\AverageKSet_\cK^+$.

Finally, a last important preliminary result which we will be using later concerns the regret of such predictors:
  the optimal classifier $\AverageKSet_\cK^*$ is indeed the average-$\cK$ classifier with the lowest error rate.
This is a more general version of a result shown by \citet[Proposition 4]{Denis2017} and works even when their continuity assumption does not hold.

\begin{proposition}
  \label{thm:adaptive-set-prediction-excess-risk}
  For any deterministic average-$\cK$ predictor $\AverageKSet_\cK$, \ie having exactly an average set size of $\cK$, its regret compared to the optimal predictor $\AverageKSet_\cK^*$ is equal to
  \begin{equation*}
    \Error(\AverageKSet_\cK) - \Error(\AverageKSet_\cK^*) = \sum_k \E[X]{ \left| \eta_k(X) - G_\eta^{-1}(\cK) \right| \ind{k \in \AverageKSet_\cK(X) \SymDiff \AverageKSet_\cK^*(X)} }
  \end{equation*}
  where $\SymDiff$ is the symmetric difference between two sets.
\end{proposition}

\begin{proof}
  The proof is in fact the same as the one of \cite[Proposition 4]{Denis2017} even in this more general setting.
  Note that, by construction, $\AverageKSet_\cK^*$ is indeed an average-$\cK$ classifier as
  \begin{equation*}
    \Info( \AverageKSet_\cK^* )
    = \Info( \AverageKSet_\cK^+ ) + \Info(\widetilde{\AverageKSet}_\cK^{=})
    = \Info( \AverageKSet_\cK^+ ) + \left( \cK - \Info( \AverageKSet_\cK^+ ) \right)
    = \cK .
  \end{equation*}

  Let $\lambda = G_\eta^{-1}(\cK)$.
  The difference of the pointwise error rate for every $x \in \X$ is equal to
  \begin{align*}
    \Error( \AverageKSet_\cK ; x) - \Error( \AverageKSet_\cK^* ; x)
    &= \left( 1 - \sum_{k \in \AverageKSet_\cK(x)} \eta_k(x) \right) - \left( 1 - \sum_{k \in \AverageKSet_\cK^*(x)} \eta_k(x) \right) \\
    &= \sum_{k \in \AverageKSet_\cK^* \setminus \AverageKSet_\cK (x)} \eta_k(x) - \sum_{k \in \AverageKSet_\cK \setminus \AverageKSet_\cK^* (x)} \eta_k(x) \\
    &=  \sum_{k \in \AverageKSet_\cK^* \setminus \AverageKSet_\cK (x)} \left( \eta_k(x) - \lambda \right) + \sum_{k \in \AverageKSet_\cK \setminus \AverageKSet_\cK^* (x)} \left( \lambda - \eta_k(x) \right) \\
    &\qquad + \lambda \left( \left| \AverageKSet_\cK^* \setminus \AverageKSet_\cK (x) \right| - \left| \AverageKSet_\cK \setminus \AverageKSet_\cK^* (x) \right| \right) \\
    &= \sum_{k \in \AverageKSet_\cK^* \SymDiff \AverageKSet_\cK (x)} \left| \eta_k(x) - \lambda \right|  + \lambda \left( \left| \AverageKSet_\cK^*(x) \right| - \left| \AverageKSet_\cK(x) \right| \right).
  \end{align*}

  Taking the expectation, we then have
  \begin{align*}
    \Error(\AverageKSet_\cK) - \Error(\AverageKSet_\cK^*)
    &= \E[X]{ \sum_{k \in \AverageKSet_\cK^* \SymDiff \AverageKSet_\cK (x)} \left| \eta_k(x) - \lambda \right| } + \lambda ( \Info(\AverageKSet_\cK^*) - \Info(\AverageKSet_\cK) ).
  \end{align*}
  The last term is equal to zero as the average set size of both set-valued classifiers is equal to $\cK$.
  This concludes the proof.
\end{proof}

Note that, as a top-$K$ classifier is also an average-$K$ classifier, applying this result to the case where $\cK = K$, this also implies that the error rate of the optimal average-$K$ classifier is necessarily less than or equal to that of the optimal top-$K$ classifier.

\subsection{Contributions}

The main goal of this paper is to characterize both theoretically and practically the usefulness of average-$K$ compared to top-$K$ classification.
Specifically, we will compare the two following optimal predictors:
\begin{enumerate}
\item optimal top-$K$ classifier (\Cref{sec:top-k-formulation}):
  \begin{equation}
    \TopKSet_K^*(x) = \enscond{ \sigma_x(k) }{ k \in \ens{1,\dots,K} },
  \end{equation}
  where $\sigma_x$ is a permutation of $\ens{1,\dots,C}$ such that $\eta_{\sigma_x(1)} (x) \geq \eta_{\sigma_x(2)} (x) \geq \dots \geq \eta_{\sigma_x(C)} (x)$;
\item optimal average-$\cK$ classifier (\Cref{sec:average-K-formulation}):
  \begin{equation}
    \AverageKSet_\cK^*(x) = \enscond{ k \in \Y }{ \eta_k(x) > \lambda_\cK },
  \end{equation}
  where $\lambda_\cK = G_\eta^{-1}(\cK)$ (for the complete expression of this classifier, see \Crefparenthesis{eq:optimal-adaptive-top-k-classifier}).
\end{enumerate}
We will compare the two classifiers for the same fixed average set size of $\cK=K$. 
We tackle the following  questions:
\begin{enumerate}
\item Theoretical: For which problems is a top-$K$ strategy $\TopKSet_K^*(x)$  optimal w.r.t. the adaptive counterpart $\AverageKSet_K^*(x)$?
\item Theoretical: For which problems does the average-$K$ strategy $\AverageKSet_K^*(x)$ achieve a lower error rate than the fixed set size one $\TopKSet_K^*(x)$ and by how much?
\item Theoretical and Practical: How can one build consistent estimators $\hat{\TopKSet}_K(x)$ and $\hat{\AverageKSet}_K(x)$ for these two strategies?
\item Practical: Given a data set, will an adaptive strategy $\hat{\AverageKSet}_K(x)$ generally outperform the top-$K$ one $\hat{\TopKSet}_K(x)$ in practice?
\end{enumerate}
We will address these questions in order.

In \Cref{sec:introductory-examples}, we start by presenting toy examples to show cases where top-$K$ and average-$K$ are useful.
These examples serve as illustrations for the subsequent sections.

In \Cref{sec:top-k-optimality}, we study when top-$K$ works as well as average-$K$.
We characterize the problems for which this holds by giving an interpretable necessary and sufficient condition involving $\Pr_X$ and $\eta_k$, connected to a notion of heterogeneity of the task ambiguity.
We show that this notion of heterogeneity is not  captured by the variance of the ambiguity but rather depends on the overlap of the distributions of the ordered conditional probabilities $\tilde{\eta}_K$ and $\tilde{\eta}_{K+1}$.

In \Cref{sec:adaptive-svp-usefulness}, we characterize the settings when $\AverageKSet_K^*$ provides a gain over $\TopKSet_K^*(x)$ and quantify by how much.
In particular, we will provide a lower bound on this improvement which highlights for which problems the adaptive strategy will be most useful.
This lower bound is expressed with a quantification of the heterogeneity of the ambiguity which we call \emph{straddle strength}.
This quantity captures both the \emph{weight} and the \emph{magnitude} of the previously-mentioned overlap of the distributions of the ordered conditional probabilities.

Next, in \Cref{sec:estimation-procedures}, we provide estimation procedures and prove their consistency, both for top-$K$ and average-$K$ classification.
These procedures are natural and intuitive, they are based on plug-in rules and strongly proper losses minimization.
To prove this result, we derive new plug-in regret bounds and generalize strongly proper losses to the multi-class case.
In particular, this allows us to prove that building top-$K$ and average-$K$ classifiers from the scores of a model learned by minimizing the negative log-likelihood is consistent.
Note that, more generally, the results of \Cref{sec:strongly-proper-losses} on strongly proper losses are of interest on their own for multi-class classification beyond top-$K$ and average-$K$ classification.

Finally, in \Cref{sec:experiments}, using these procedures, we study how the previous theoretical findings apply to real-world image data sets.
First, we study the estimation procedures coupled with classic neural network training and show their efficiency under various conditions.
Then, we report on experiments carried out on large-scale data sets to evaluate the usefulness of the adaptive average-$K$ strategy on concrete problems.
Surprisingly, in practice, average-$K$ \emph{always} outperforms top-$K$, even when the amount of training data is limited.

\section{Introductory Examples}
\label{sec:introductory-examples}

In this section, we present two toy examples and study how top-$K$ and average-$K$ perform in these cases.
These examples are deliberately contrived to train the reader's intuition.
We'll use (and explore variants of) them in  subsequent sections to illustrate both definitions and results.

For these examples, we consider applications with $C=6$ classes, where the goal, given an input, is to return a small set containing the proper class.
We denote the different classes using letters from $A$ to $F$.

\begin{figure}
  \centering
  \begin{subfigure}{.15\textwidth}
    \foreach \i in {1,2,3} {
      \begin{subfigure}{\textwidth}
        \centering
        \includegraphics[width=\textwidth]{imgs/introductory_examples/ex1_\i} \\
      \end{subfigure}
    }
    \begin{equation*}
      \vdots
    \end{equation*}
    \begin{subfigure}{\textwidth}
      \centering
      \includegraphics[width=\textwidth]{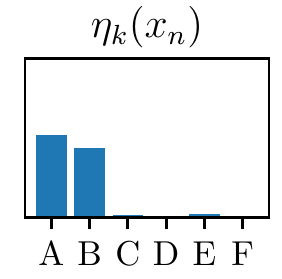} \\
    \end{subfigure}
  \end{subfigure}
  \hfill
  \begin{subfigure}{.15\textwidth}
    \foreach \i in {1,2,3} {
      \begin{subfigure}{\textwidth}
        \centering
        \includegraphics[width=\textwidth]{imgs/introductory_examples/ex1_\i_ordered} \\
      \end{subfigure}
    }
    \begin{equation*}
      \vdots
    \end{equation*}
    \begin{subfigure}{\textwidth}
      \centering
      \includegraphics[width=\textwidth]{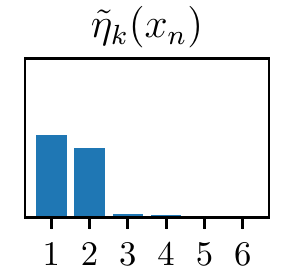} \\
    \end{subfigure}
  \end{subfigure}
  \hfill
  \begin{subfigure}{.5\textwidth}
    \begin{subfigure}{\textwidth}
      \centering
      \input{plots/introductory_examples/ex1_error_rate_reduction_bar_plot.pgf}%
    \end{subfigure}
    \newline
    \vspace{.5cm}
    \newline
    \begin{subfigure}{\textwidth}
      \centering
      \input{plots/introductory_examples/ex1_set_size_distribution_K_2.pgf}%
    \end{subfigure}
  \end{subfigure}
  \caption{
    Illustration of Example~1.
    The first column contains samples with their associated conditional probabilities $\eta_k$ over the classes.
    The second column orders the classes by the heights of the bars in the histogram, $\tilde{\eta}_k$.
    For example, when classes $C$ and $D$ have nearly all the probability, then $\tilde{\eta}_k$ simply reprises the probabilities of those two classes, in descending order of probability.
    The top-right figure plots the error rate for  different $K$ values for top-$K$ and average-$K$.
    The bottom-right figure shows the distribution of the predicted set size of average-$K$ strategy for $K=2$.
    In this case, the classes are grouped by pairs and there  average-$K$ provides no benefit over top-$K$. A set size of 2 is best in all cases.
    See the text for more details.
  }
  \label{fig:example1}
\end{figure}

The first example, Example~1, is illustrated in \Cref{fig:example1}.
Here, the 6 classes are grouped by pairs: $\{A, B\}$, $\{C, D\}$, and $\{E, F\}$.
For each input (say an image), two classes have high, roughly equal, probabilities with some variability while all other classes have  negligible probabilities.
This is illustrated in the first column containing some samples with their associated conditional probability $\eta_k(x) = \prcond{ Y=k }{ X=x }$: the first sample $x_1$ is ambiguous between $A$ and $B$, while the second one $x_2$ between $C$ and $D$, $x_3$ between $E$ and $F$, etc.
The second column shows these conditional probabilities reordered in descending order, denoted $\tilde{\eta}_K$.
In each case, the two most probable classes are much more likely than the other ones.
In this scenario, top-2 and average-2 have the same error rate.
Average-$2$ will always choose two classes and so of course will top-$2$.
This is illustrated in the top-right and bottom-right figures.

\begin{figure}
  \centering
  \begin{subfigure}{.15\textwidth}
    \foreach \i in {1,2,3} {
      \begin{subfigure}{\textwidth}
        \centering
        \includegraphics[width=\textwidth]{imgs/introductory_examples/ex2_\i} \\
      \end{subfigure}
    }
    \begin{equation*}
      \vdots
    \end{equation*}
    \begin{subfigure}{\textwidth}
      \centering
      \includegraphics[width=\textwidth]{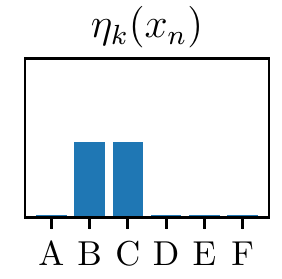} \\
    \end{subfigure}
  \end{subfigure}
  \hfill
  \begin{subfigure}{.15\textwidth}
    \foreach \i in {1,2,3} {
      \begin{subfigure}{\textwidth}
        \centering
        \includegraphics[width=\textwidth]{imgs/introductory_examples/ex2_\i_ordered} \\
      \end{subfigure}
    }
    \begin{equation*}
      \vdots
    \end{equation*}
    \begin{subfigure}{\textwidth}
      \centering
      \includegraphics[width=\textwidth]{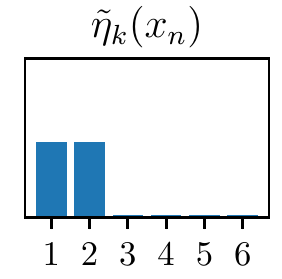} \\
    \end{subfigure}
  \end{subfigure}
  \hfill
  \begin{subfigure}{.5\textwidth}
    \begin{subfigure}{\textwidth}
      \centering
      \input{plots/introductory_examples/ex2_error_rate_reduction_bar_plot.pgf}%
    \end{subfigure}
    \newline
    \vspace{.5cm}
    \newline
    \begin{subfigure}{\textwidth}
      \centering
      \input{plots/introductory_examples/ex2_set_size_distribution_K_2.pgf}%
    \end{subfigure}
  \end{subfigure}
  \caption{
    Illustration of Example~2.
    The first column contains four samples with their associated conditional probabilities $\eta_k$ for each class.
    The second column reprises the probabilities  in descending order of the height of the bars in the corresponding left histogram, $\tilde{\eta}_k$, giving a probability density function of the number of classes found. In this example, for 1/3 of the inputs there is no ambiguity; for another 1/3,  there is ambiguity among two classes; and for the final 1/3, there is ambiguity among three classes.
    The top-right figure plots the error rates (vertical axis) for different values of $K$ (horizontal axis)   for both the top-$K$ and average-$K$ strategies. The orange shows the additional error when using top-$K$ compared with using average-$K$.
    The bottom-right figure shows the distribution of the set sizes of the average-$K$ strategy for $K=2$, showing that the average-$K$ strategy predicts a set size of 1, 2, or 3, each 1/3 of the time which accords with the ambiguity distribution.
    For that reason, average-2 has no errors, whereas top-2 will have an error 1/3 of the time when there is an ambiguity of 3 ($1/3 \times 1/3 = 1/9$).
    See the text for more details.
  }
  \label{fig:example2}
\end{figure}

The second example, Example~2, is illustrated in \Cref{fig:example2}.
Here, for some inputs, there is no ambiguity and
one class, $A$, has all the probability.
For other inputs, there is an ambiguity between two classes, $B$ and $C$. For yet other inputs,  there is ambiguity among three classes, $D$, $E$ and $F$.
There is an equal number of inputs for each level of ambiguity.

The figures illustrate this as follows:   the first sample $x_1$ is unambiguous and belongs to class $A$, while the second one $x_2$ has the same likelihood for $B$ and $C$ and $x_3$ the same likelihood for $D$, $E$ and $F$, etc.
In this scenario, top-$2$ reduces the error compared to top-1 classification but average-$2$ reduces the error rate further compared to top-$2$.
In fact, as shown in the bottom-right figure, when fixing $K=2$, the average-2 classifier predicts a balanced number of set of sizes 1, 2 and 3.
For each of the groups, it predicts exactly the correct number of classes present.
That is why average-2 has a zero error rate.

These two very simple examples offer a preliminary intuition to the question: when is  average-$K$ useful? Intuitively, average-$K$ has a lower error rate than top-$K$ when it would be best to guess fewer than $K$ classes sometimes and more than $K$ at other times. That is, the distribution of predicted set sizes should {\em straddle} $K$.
Formalizing and quantifying the utility of average-$K$ is the aim of the next two sections.

\section{When Is Top-$K$ Classification Optimal?}
\label{sec:top-k-optimality}

In this section, we first focus on the optimality of top-$K$ classification with respect to the adaptive strategy.
In other words, when is top-$K$ sufficient and when would the adaptive strategy improve on it?

In terms of error rates, as we will see in this section, we have the following relationship between the error rate of classic top-1 classification $\Error(S_1^*)$, top-$K$ classification $\Error(\TopKSet_K^*)$ and average-$K$ classification $\Error(\AverageKSet_K^*)$:
\begin{equation*}
  \Error(\AverageKSet_K^*) \leq \Error(\TopKSet_K^*) \leq \Error(S_1^*).
\end{equation*}
A first instructive question is: when is top-$K$ useful compared to top-1 classification?
The answer is found by analyzing the error rate of top-$K$ which is given by
\begin{equation*}
  \Error(\TopKSet_K^*) = 1 - \sum_{k \leq K} \E[X]{\tilde{\eta}_k(X)}.
\end{equation*}
As we can see, the important quantity here is the average of ambiguities $\E[X]{\tilde{\eta}_k(X)}$.
Obviously, if there is no ambiguity, top-1 classification is optimal.
The benefit of top-$K$ is solely dependent on  \emph{average ambiguities}.

\subsection{Adaptive Gain}

To study the optimality of top-$K$, we analyze the \emph{adaptive gain} $\Delta_K$ defined as
\begin{equation*}
  \Delta_K \eqdef \Error(\TopKSet_K^*) - \Error(\AverageKSet_K^*).
\end{equation*}
To simplify the notation, we will denote the adapted threshold of $\AverageKSet_K^*$ as $\lambda_K$, \ie
\begin{equation}
  \label{eq:average-K-threshold}
  \lambda_K \eqdef G_\eta^{-1}(K).
\end{equation}
Because top-$K$ classifiers also have an average set size of $K$, using \Cref{thm:adaptive-set-prediction-excess-risk}, we have
\begin{equation}
  \label{eq:adaptive-gain}
  \Delta_K = \sum_k \E[X]{ \left| \eta_k(X) - \lambda_K \right| \ind{k \in \TopKSet_K^*(X) \SymDiff \AverageKSet_K^*(X)} }
\end{equation}
which is clearly non-negative.
Adaptive gain thus quantifies the benefit of average-$K$ over top-$K$.


The following proposition is the main result of this section, it gives a natural necessary and sufficient condition for the usefulness/uselessness of the average-$K$ strategy compared to top-$K$.
In particular, it is expressed in terms of the support\footnote{The definition we consider here is not the classical definition of the support of a distribution. The definition we give here is equal to the interior of the smallest (closed) segment containing the support, in the classical sense, of the distribution.} of the distributions of $\tilde{\eta}_K$ and $\tilde{\eta}_{K+1}$ denoted respectively $\support{ \tilde{\eta}_K }$ and $\support{ \tilde{\eta}_{K+1} }$.
We define the support of $\tilde{\eta}_k$ for all $k$ as
\begin{equation*}
  \support{ \tilde{\eta}_k } \eqdef \interior{ \preimage[(0,1)]{F_{\tilde{\eta}_k}} } ,
\end{equation*}
\ie the interior of the inverse image of the interval $(0,1)$ under the cumulative density function (CDF) of $\tilde{\eta}_k$, defined as
\begin{equation*}
  \preimage[(0,1)]{F_{\tilde{\eta}_k}} \eqdef \enscond{ t \in [0,1] }{ F_{\tilde{\eta}_k}(t) = \pr[X]{ \tilde{\eta}_K(X) \leq t } \in (0,1) } .
\end{equation*}
Intuitively, the support of $\tilde{\eta}_k$ is the interval between the lowest  and the highest probabilities that the $k$th most likely class can take (the range of $\tilde{\eta}_k$)  over all samples.
In Example~1 (\Cref{fig:example1}), the most likely class always has a probability slightly more than $1/2$, so the range of $\tilde{\eta}_1$ hovers just above $1/2$.
In Example~2  (\Cref{fig:example2}), the most likely class can  have probability $1$, $1/2$ or  $1/3$, so   $\tilde{\eta}_1$ will span the interval $1/3$ to $1$.
The second most likely class in Example~2 has a probability ($\tilde{\eta}_2$) that ranges from $0$ to $1/2$. 

We can now present the main result of this section, \Cref{thm:characterization-uselessness-set-prediction}.

\begin{proposition}[Top-$K$ optimality characterization]
  \label{thm:characterization-uselessness-set-prediction}
  For a fixed $K<C$, the following statements are equivalent:
  \begin{enumerate}
  \item \label{enum:null-regret} The adaptive gain of using average-$K$ is zero:
    \begin{equation*}
      \Delta_K = 0,
    \end{equation*}
  \item \label{enum:support-gap} There is a gap between the support of the distributions of the ranked conditional probabilities $\tilde{\eta}_K$ and $\tilde{\eta}_{K+1}$:
    \begin{equation*}
      \exists \lambda \in [0,1] \; \mid \; \left( \tilde{\eta}_K(X) \geq \lambda \geq \tilde{\eta}_{K+1}(X) \emph{ almost everywhere} \right),
    \end{equation*}
  \item \label{enum:support-overlap} There is no overlap between the support of the distributions of the ranked conditional probabilities $\tilde{\eta}_K$ and $\tilde{\eta}_{K+1}$:
    \begin{equation*}
      \support{ \tilde{\eta}_K } \, \cap \, \support{ \tilde{\eta}_{K+1} } = \emptyset.
    \end{equation*}
  \end{enumerate}
\end{proposition}

Note that, in the absence of ties,  statement \ref{enum:null-regret} is equivalent to having top-$K$ and average-$K$ predict the same sets:
\begin{equation*}
  \TopKSet_K^*(X) = \AverageKSet_K^*(X) \text{ almost everywhere}.
\end{equation*}
Before providing the proof of \Cref{thm:characterization-uselessness-set-prediction} (in \Cref{sec:proof-characterization-uselessness-set-prediction}), we discuss in the next subsection its interpretation and consequences.

\subsection{Discussion of \Cref{thm:characterization-uselessness-set-prediction}: The Notion of Overlap of Ambiguity}

Statements \ref{enum:support-gap} and \ref{enum:support-overlap} of \Cref{thm:characterization-uselessness-set-prediction} highlight that the overlap (or its absence) between the cumulative density functions (hereafter called ``density overlap'' or just ``overlap'') of the different $\tilde{\eta}_k$ is an important characteristic of the problem.
In this subsection, we revisit the introductory examples of \Cref{sec:introductory-examples} and, based on those, show the kind of  ambiguity heterogeneity that  adaptive gain is related to.

\begin{figure}[t]
  \centering
  \begin{subfigure}{.45\textwidth}
    \centering
    \includegraphics[width=\textwidth]{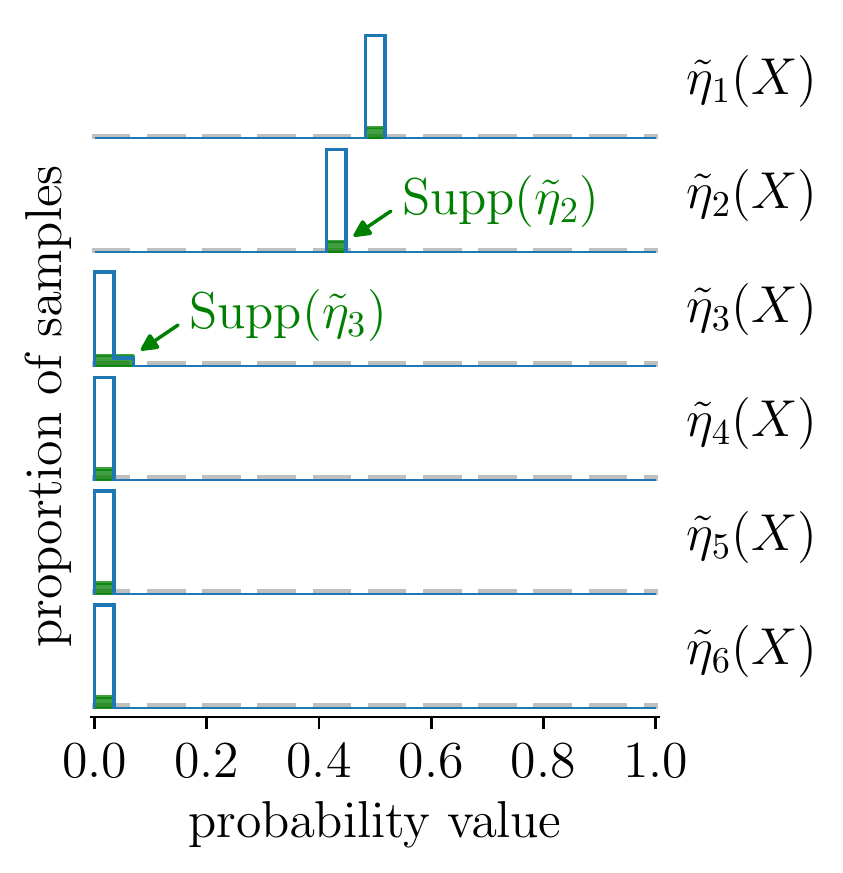}
    \caption{Example~1 analysis}
  \end{subfigure}
  \hfill
  \begin{subfigure}{.45\textwidth}
    \centering
    \includegraphics[width=\textwidth]{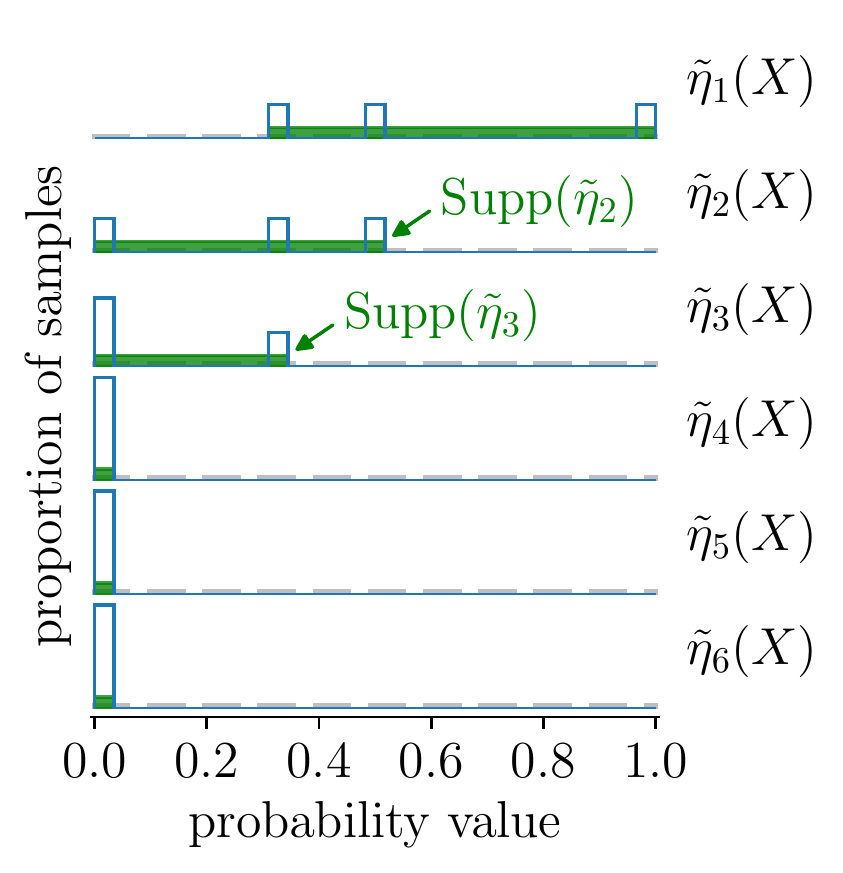}
    \caption{Example~2 analysis}
  \end{subfigure}
  \caption{
    Distribution of $\tilde{\eta}_k(X)$ for all $k$ for Example~1 and Example~2.
    The supports $\support{\tilde{\eta}_k}$ are shown in the bold green line segments in both cases.
    Focusing on $K=2$, there is no overlap between $\support{\tilde{\eta}_2}$ and $\support{\tilde{\eta}_3}$ in Example~1. By contrast, such an overlap exists in Example~2.
    For Example~2,  sometimes the second most likely class is roughly as likely as the third most likely class and sometimes their likelihoods are very different.
    When their likelihoods are similar, average-$2$ will return the top three classes.
    When the two highest probability classes have similar and relatively high probabilities, then average-$2$ will return only two classes.
    When the two lower probability classes have very small probabilities, then average-$2$ will return only one class.
  }
  \label{fig:ranked-cond-prob-distributions}
\end{figure}

\Cref{fig:ranked-cond-prob-distributions} shows the distribution of $\tilde{\eta}_k(X)$ for all $k$, both for Example~1 and Example~2.
In Example~1, where top-$K$ is optimal, there is no overlap in  of the bold green horizontal line segments among the non-zero probability supports. By contrast, in Example~2, the second most probable class can range from a probability of $0$ to $1/2$ and the third most probable class can range from a probability of $0$ to $1/3$ reflected in an overlap of the horizontal bold green line segments in \Cref{fig:ranked-cond-prob-distributions}.


As noted in the caption of \Cref{fig:ranked-cond-prob-distributions}, in Example~2, sometimes the second most likely class is roughly as likely as the third most likely class and sometimes their likelihoods are very different.
When their likelihoods are similar, average-$2$ will return the top three classes.
When the two lower probability classes have very small probabilities, then average-$2$ will return only one class.
Thus, average-$2$ gives fewer errors than top-$2$ for Example~2.

These simple examples illustrate that the notion of heterogeneity introduced in \Cref{thm:characterization-uselessness-set-prediction}, \ie the lack of overlap of the distributions of $\tilde{\eta}_K$ and $\tilde{\eta}_{K+1}$ characterizes cases when  top-$K$ is optimal.
By contrast and somewhat surprisingly, a measure of variability such as the variance, which would have been a natural candidate at first glance, is actually not the appropriate quantity to measure. Even with some variability of the $\tilde{\eta}_k$, top-$K$ can remain optimal.
We will quantify the exact gain of average-$K$ over top-$K$ in \Cref{sec:adaptive-svp-usefulness}.

\subsection{Proof of \Cref*{thm:characterization-uselessness-set-prediction}}
\label{sec:proof-characterization-uselessness-set-prediction}

Let us now prove \Cref{thm:characterization-uselessness-set-prediction}.
To do so, we will prove the different equivalences separately.

\begin{proof}[Proof of equivalence \ref*{enum:null-regret} $\Leftrightarrow$ \ref*{enum:support-gap}]
  We have the following equivalence:
  \begin{equation}
    \label{eq:null-regret-equivalence}
    \Delta_K = 0
    \; \Leftrightarrow \;
    \pr[X]{ \tilde{\eta}_K(X) \geq \lambda_K \geq \tilde{\eta}_{K+1}(X) } = 1 .
  \end{equation}
  From the expression of $\Delta_K$ in \Cref{eq:adaptive-gain}, we have
  \begin{align*}
    \Delta_K = 0
    &\Leftrightarrow \; \forall k, \; \left| \eta_k(X) - \lambda_K \right| \ind{k \in \TopKSet_K^*(X) \SymDiff \AverageKSet_K^*(X)} = 0 \quad a.e. \\
    &\Leftrightarrow \; \forall k, \; \left( k \notin \TopKSet_K^*(X) \SymDiff \AverageKSet_K^*(X) \text{ or } \eta_k(X) = \lambda_K \right) \quad a.e.
  \end{align*}
  By reordering by decreasing order of $\eta_k(X)$, we then have
  \begin{align*}
    \Delta_K = 0
    &\Leftrightarrow \; \forall k, \left( \left\{ \begin{aligned}
          &k \leq K \\
          &\tilde{\eta}_k(X) > \lambda_K
        \end{aligned} \right.
          \text{ or }
          \left\{ \begin{aligned}
              &k > K \\
              &\tilde{\eta}_k(X) < \lambda_K
            \end{aligned} \right.
                \text{ or } \tilde{\eta}_k(X) = \lambda_K \right) a.e. \\
    &\Leftrightarrow \; \forall k, \left( \left\{ \begin{aligned}
      &k \leq K \\
      &\tilde{\eta}_k(X) \geq \lambda_K
    \end{aligned} \right.
        \text{ or }
        \left\{ \begin{aligned}
            &k > K \\
            &\tilde{\eta}_k(X) \leq \lambda_K
          \end{aligned} \right.
              \right) a.e. \\
    &\Leftrightarrow \; \left( \tilde{\eta}_K(X) \geq \lambda_K \text{ and } \tilde{\eta}_{K+1}(X) \geq \lambda_K \right) a.e. \\
    &\Leftrightarrow \; \pr[X]{ \tilde{\eta}_K(X) \geq \lambda_K \geq \tilde{\eta}_{K+1}(X) } = 1 .
  \end{align*}
  Using this result of \Cref{eq:null-regret-equivalence}, the implication \ref*{enum:null-regret} $\Rightarrow$ \ref*{enum:support-gap} is direct as we can take $\lambda = \lambda_K$.

  To prove the converse \ref*{enum:support-gap} $\Rightarrow$ \ref*{enum:null-regret}, we denote $\Lambda$ the following set:
  \begin{equation*}
    \Lambda \eqdef \enscond{ \lambda \in [0,1] }{ \pr[X]{ \tilde{\eta}_{K+1}(X) \leq \lambda \leq \tilde{\eta}_K(X) } = 1 }.
  \end{equation*}
  Statement \ref*{enum:support-gap} implies that $\Lambda \neq \emptyset$.
  Moreover, due to the right-continuity of the CDF of $\tilde{\eta}_{K+1}(X)$, $\Lambda$ has a minimum which we denote $\min \Lambda$.

  Then, for all $\lambda \in \Lambda$, we have $G_\eta(\lambda) \leq K$ as 
  \begin{align*}
    \lambda \in \Lambda
    &\Rightarrow \; \left\{ \begin{aligned}
        &\pr[X]{ \tilde{\eta}_K(X) \geq \lambda } = 1 \\
        &\pr[X]{ \tilde{\eta}_{K+1}(X) > \lambda } = 0
      \end{aligned} \right. \\
    &\Rightarrow \; \left\{ \begin{aligned}
        &\forall k \leq K, \; \pr[X]{ \tilde{\eta}_k(X) \geq \lambda } = 1 \\
        &\forall k > K, \; \pr[X]{ \tilde{\eta}_k(X) > \lambda } = 0
      \end{aligned} \right. \\
    &\Rightarrow \; G_\eta(\lambda) = \sum_k \pr[X]{ \tilde{\eta}_k(X) > \lambda } \\
    &\phantom{\Rightarrow \; G_\eta(\lambda)} \leq \sum_{k \leq K} \pr[X]{ \tilde{\eta}_k(X) \geq \lambda } \\
    &\phantom{\Rightarrow \; G_\eta(\lambda)} = K .
  \end{align*}
  Thus, we have, by definition of $\lambda_K$, $\lambda_K \leq \min \Lambda$.

  Moreover, we have that
  \begin{align*}
    \lambda' < \min \Lambda
    &\Rightarrow \; \left\{ \begin{aligned}
        & \pr[X]{ \tilde{\eta}_K(X) > \lambda' } = 1 \text{ as } \lambda' < \lambda \in \Lambda \\
        & \pr[X]{ \tilde{\eta}_{K+1}(X) > \lambda' } > 0 \text{ as, otherwise, } \lambda' \in \Lambda
      \end{aligned} \right. \\
    &\Rightarrow \; \left\{ \begin{aligned}
        & \forall k \leq K, \pr[X]{ \tilde{\eta}_k(X) > \lambda' } = 1 \\
        & \pr[X]{ \tilde{\eta}_{K+1}(X) > \lambda' } > 0
      \end{aligned} \right. \\
    &\Rightarrow \; G_\eta(\lambda') \geq K + \pr[X]{ \tilde{\eta}_{K+1}(X) > \lambda' } > K .
  \end{align*}
  Thus, necessarily, $\lambda_K \geq \min \Lambda$ as we have $G_\eta(\lambda_K) \leq K$.

  We have shown that $\lambda_K = \min \Lambda$ and, as a result, $\lambda_K \in \Lambda$.
  We can then conclude the proof using the equivalence of \Cref{eq:null-regret-equivalence}.
\end{proof}

\begin{proof}[Proof of equivalence \ref*{enum:support-gap} $\Leftrightarrow$ \ref*{enum:support-overlap}]
  The preimages of the interval (0,1) under the CDFs are intervals, open, closed or half-open and can be reduced to the empty set or a singleton.
  The supports are thus either open intervals or empty sets.
  In general, they are open intervals and we denote them as
  \begin{equation*}
    \support{\tilde{\eta}_K} = (a_K,b_K)
    \quad \text{and} \quad
    \support{\tilde{\eta}_{K+1}} = (a_{K+1},b_{K+1})
  \end{equation*}
  with $a_K<b_K$ and $a_{K+1}<b_{K+1}$.
  In the case where $\support{\tilde{\eta}_K} = \emptyset$, all the distribution of $\tilde{\eta}_K$ is concentrated in a single point which we denote $a_K$.
  Similarly, in the case where $\support{\tilde{\eta}_{K+1}} = \emptyset$, all the distribution of $\tilde{\eta}_{K+1}$ is concentrated in a single point which we denote $b_{K+1}$.
  Using these notations, we then have
  \begin{equation*}
    \support{ \tilde{\eta}_K } \cap \support{ \tilde{\eta}_{K+1} } = \emptyset
    \quad \Leftrightarrow \quad
    b_{K+1} \leq a_K
  \end{equation*}
  
  Starting with the first implication, if such a $\lambda$ exists, it satisfies
  \begin{equation*}
    \left\{ \begin{aligned}
        &\pr[X]{ \tilde{\eta}_K(X) \geq \lambda } = 1 \\
        &\pr[X]{ \tilde{\eta}_{K+1}(X) > \lambda } = 0
      \end{aligned} \right.
    \quad \Rightarrow \quad
    \left\{ \begin{aligned}
        & \lambda \leq a_K \\
        & \lambda \geq b_{K+1}
      \end{aligned} \right.
    \quad \Rightarrow \quad
    b_{K+1} \leq a_K .
  \end{equation*}

  Conversely, starting from $b_{K+1} \leq a_K$, we take $\lambda = a_K$.
  We then have
  \begin{equation*}
    \pr[X]{ \tilde{\eta}_K(X) < \lambda } = F_{\tilde{\eta}_K}(a_K) - \pr[X]{ \tilde{\eta}_K(X) = a_K } = 0
    \; \Rightarrow \;
    \pr[X]{ \tilde{\eta}_K(X) \geq \lambda } = 1
  \end{equation*}
  and
  \begin{equation*}
    F_{\tilde{\eta}_{K+1}}(\lambda) = \pr[X]{ \tilde{\eta}_{K+1}(X) \leq \lambda } = 1
  \end{equation*}
  From that, we can conclude
  \begin{equation*}
    \pr[X]{ \tilde{\eta}_{K+1}(X) \leq \lambda \leq \tilde{\eta}_K(X) } = \pr[X]{ \tilde{\eta}_{K+1}(X) \leq \lambda } \pr[X]{ \tilde{\eta}_K(X) \geq \lambda } = 1.
  \end{equation*}
\end{proof}

\section{Quantifying the Usefulness of Average-$K$}
\label{sec:adaptive-svp-usefulness}

This section quantifies how much we can reduce the error rate using the  average-$K$ strategy for some problems compared with top-$K$.
That is the \emph{adaptive gain}.
For this, we will quantify the heterogeneity of the ambiguity of the problem using an interpretable quantity, the \emph{straddle strength}, expressed in terms of $\Pr_{X,Y}$ (thus $\Pr_X$ and $\eta$).

\subsection{Illustrative Examples}

Before stating the theoretical result, we introduce two variations of Example~2 in which the top-2 error rate does not change but the average-2 error rate, instead of going to $0$, is nearly the same as the top-2 rate.
Recall that, in Example~2, there are three zones: zone~1 with class $A$ unambiguous, zone~2 with class $B$ and $C$ equally ambiguous, and zone~3 with class $D, E$, and $F$ equally ambiguous.
Each of these zones is equally probable, their weight are equal to $w_1=w_2=w_3=1/3$.

\begin{figure}[t]
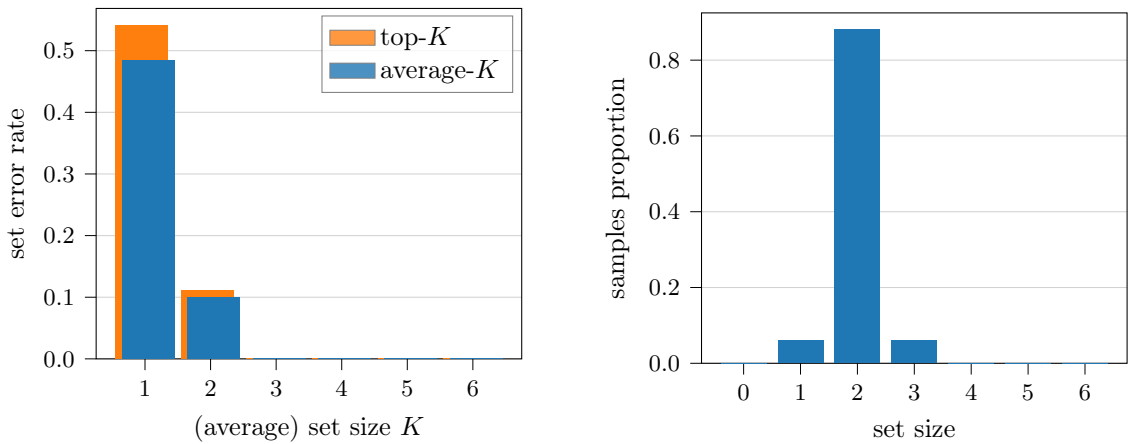

  \centering
  \begin{subfigure}{.475\textwidth}
    \centering
    \input{plots/introductory_examples/ex3_error_rate_reduction_bar_plot.pgf}%
  \end{subfigure}
  \hfill
  \begin{subfigure}{.475\textwidth}
    \centering
    \input{plots/introductory_examples/ex3_set_size_distribution_K_2.pgf}%
  \end{subfigure}
  \caption{
    Example~3, a slight modification of Example~2 where, instead of having three zones of equal weight ($w_1 = w_2 = w_3 = 1/3$), these weights are adapted in order to keep $w_3=1/3$ but to unbalance $w_1$ and $w_2$ such to have a small $w_1 = 3/100$ and a large $w_2 = 2/3 - 3/100$.
    The left figure plots the error rate of top-$K$ and average-$K$ for different $K$ values.
    The right figure shows the distribution of the predicted set size of average-$K$ strategy for $K=2$.
    In this case, average-$2$ predicts the same sets as top-$2$ for the vast majority of inputs.
    See the text for more details.
  }
  \label{fig:ex3-error-rate}
\end{figure}

First, we derive Example~3 for which these weights have been adapted.
In particular, we keep $w_3=1/3$ to preserve the top-2 error rate.
$w_1$ and $w_2$ are however unbalanced: $w_1 = 3/100$ and $w_2 = 2/3 - 3/100$.
In this case, average-$2$ can  improve upon top-$2$ for some input samples, but those inputs occur rarely.
In \Cref{fig:ex3-error-rate}, we can see that, in most cases, average-2 predicts the same sets as top-2, thus average-2 has a very small adaptive gain over top-2. 

\begin{figure}[p]
  \centering
  \begin{subfigure}{.15\textwidth}
    \foreach \i in {1,2,3} {
      \begin{subfigure}{\textwidth}
        \centering
        \includegraphics[width=\textwidth]{imgs/introductory_examples/ex4_\i} \\
      \end{subfigure}
    }
    \begin{equation*}
      \vdots
    \end{equation*}
    \begin{subfigure}{\textwidth}
      \centering
      \includegraphics[width=\textwidth]{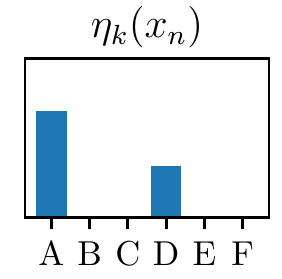} \\
    \end{subfigure}
  \end{subfigure}
  \hfill
  \begin{subfigure}{.15\textwidth}
    \foreach \i in {1,2,3} {
      \begin{subfigure}{\textwidth}
        \centering
        \includegraphics[width=\textwidth]{imgs/introductory_examples/ex4_\i_ordered} \\
      \end{subfigure}
    }
    \begin{equation*}
      \vdots
    \end{equation*}
    \begin{subfigure}{\textwidth}
      \centering
      \includegraphics[width=\textwidth]{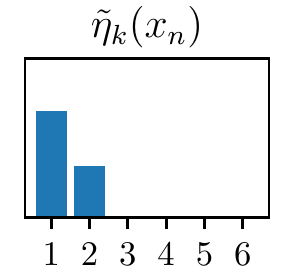} \\
    \end{subfigure}
  \end{subfigure}
  \hfill
  \begin{subfigure}{.5\textwidth}
    \begin{subfigure}{\textwidth}
      \centering
      \input{plots/introductory_examples/ex4_error_rate_reduction_bar_plot.pgf}%
    \end{subfigure}
    \newline
    \vspace{.5cm}
    \newline
    \begin{subfigure}{\textwidth}
      \centering
      \input{plots/introductory_examples/ex4_set_size_distribution_K_2.pgf}%
    \end{subfigure}
  \end{subfigure}
  \caption{
    Illustration of Example~4 which is a modification of Example~2  where class $A$ is dominant, but, instead of being unambiguous, it has a probability close to $2/3$ while one of the other classes has a probability around $1/3$ as shown in the  the first and last rows.
    The top-right figure plots the error rate for different $K$ values for top-$K$ and average-$K$.
    The bottom-right figure shows the distribution of the predicted set size of average-$K$ strategy for $K=2$.
    In this case, the sets predicted by average-$2$ are exactly the same as in Example~2, however, here, they reduce the error rate only marginally compared with top-$K$.
    See the text for more details.
  }
  \label{fig:example4}
\end{figure}

The second variant, Example~4, is shown in \Cref{fig:example4}.
It modifies zone~1 where class $A$ is dominant as follows: class $A$ has a probability close to $2/3$ and one other class has the rest of the probability, \ie a little less than $1/3$.
This does not change the top-2 error rate.
The weights $w_1,w_2,w_3$ are kept equal to $1/3$.
In this case, as can be seen in \Cref{fig:example4}, average-$2$ often predicts set sizes similar to those of  Example~2.
The net result, however, is that the adaptive gain is small.

In the following section, we provide a theoretical result explaining this behavior.

\subsection{A Lower Bound on the Adaptive Gain}

\Cref{thm:regret-lower-bound} gives a lower bound on the adaptive gain $\Delta_K$.
It is based on the \emph{straddle strength} which is a quantified version of the overlap of distributions of $\tilde{\eta}_k$ studied previously.
It measures the heterogeneity of the ambiguity by quantifying how much the level of ambiguity can differ for different inputs.

\begin{definition}[Straddle strength]
  \label{def:heterogeneity-measure}
  For a fixed $K$, the \emph{straddle strength of order $k$} denoted $d_{K,k}$ is defined by
  \begin{equation*}
    d_{K,k} \eqdef \E[X,X']{ \left( \tilde{\eta}_{K+k}(X) - \tilde{\eta}_{K+1-k}(X') \right)^+ } ,
  \end{equation*}
  where $a^+ = \max (a,0)$ is the positive part function.
\end{definition}

Note that the difference is computed between conditional probabilities for different $X$ and $X'$ (otherwise, these quantities would always equal  zero as, by definition, for $k_1 < k_2$, $\tilde{\eta}_{k_1}(X) \geq \tilde{\eta}_{k_2}(X)$).
The straddle strength of order $1$, $d_{K,1}$, is equal to
\begin{equation*}
  d_{K,1} \eqdef \E[X,X']{ \left( \tilde{\eta}_{K+1}(X) - \tilde{\eta}_K(X') \right)^+ } ,
\end{equation*}
which, using conditional expectations, can be rewritten as
\begin{align*}
  d_{K,1}
  &= \pr[X,X']{ \tilde{\eta}_{K+1}(X) > \tilde{\eta}_K(X') } \\
  & \qquad \times
    \Econd[X,X']{ \tilde{\eta}_{K+1}(X) - \tilde{\eta}_K(X') }{ \tilde{\eta}_{K+1}(X) > \tilde{\eta}_K(X') } .
\end{align*}
These quantities measure a form of \emph{heterogeneity} of the ambiguity: they quantify how much the level of ambiguity can be different for different inputs.
The first term quantifies how often this occurs (the \emph{weight}). The second term quantifies the strength of this difference (the \emph{magnitude}).

The straddle strengths of higher order, $d_{K,k}$, then generalize this idea and quantify overlaps between  conditional probabilities farther from $K$, \eg for order 2, it compares $\tilde{\eta}_{K+2}$ and $\tilde{\eta}_{K-1}$, etc.

Before stating the main result, we discuss this notion of straddle strength by computing it for the illustrative examples.
In those examples, the three zones are disjoint, the expectation of \Cref{def:heterogeneity-measure} can be reduced to the following weighted sum:
\begin{equation*}
  d_{2,1} = \sum_{i=1}^3 \sum_{j=1}^3 w_i w_j \left( \tilde{\eta}_3^{(i)} - \tilde{\eta}_2^{(j)} \right)^+,
\end{equation*}
where $\tilde{\eta}_3^{(i)}$ (resp. $\tilde{\eta}_2^{(j)}$) is the constant value of $\tilde{\eta}_3(X)$ in zone $i$ (resp. of $\tilde{\eta}_2(X)$ in zone $j$).
Every term of this sum is zero except when $i=3$ and $j=1$.
It is thus equal to
\begin{equation*}
  d_{2,1} = \underbrace{ \vphantom{\tilde{\eta}_3^{(3)}} w_1 w_3}_{\text{weight}} \times \underbrace{ ( \tilde{\eta}_3^{(3)} - \tilde{\eta}_2^{(1)} ) }_{\text{magnitude } \delta \geq 0}.
\end{equation*}
This decomposition highlights the general idea of $d_{K,1}$: it captures both how often  $\tilde{\eta}_{K+1}$ has a higher value than $\tilde{\eta}_K$ for different inputs (weight) and, when it occurs, the magnitude of the difference.
In the two variations of Example~2, \ie Example~3 and~4, one of these quantities at a time is changed: in Example~3 the weight is reduced (because there are fewer unambiguous input samples to balance the 3-way ambiguous inputs) whereas in Example~4 it is the magnitude (because the formerly unambiguous samples are now partly ambiguous).
\Cref{tab:heterogeneity-measure-table} provides the detail of these quantities.

\begin{table}[t]
  \centering
  \begin{tabular}{lccccc}
  \toprule
  & \multicolumn{2}{c}{error rates} & \multicolumn{3}{c}{heterogeneity measure} \\
  \cmidrule(r){2-3} \cmidrule(r){4-6}
  \textbf{Ex. \#} & top-$2$ & avg-$2$ &    weight $w_1 w_3$ &  magnitude $\delta$ &  measure $d_{2,1}$ \\
  \midrule
  \textbf{Ex. 2} & 11\% & 0\% & $1/9$ &   $1/3$ &              $1/27 \approx 0.037$ \\
  \textbf{Ex. 3} & 11\% & 10\% & $\epsilon$ &             $1/3$ &              $\epsilon/3$ \\
  \textbf{Ex. 4} & 11\% & 11\% & $1/9$ &              $\epsilon$ &  $\epsilon/9$ \\
  \bottomrule
\end{tabular}

  \caption{
    Decomposition of the heterogeneity measure for the given examples.
    We can build examples with $\epsilon$ arbitrarily small, in Examples 3 and 4, we took $\epsilon = 0.01$.
  }
  \label{tab:heterogeneity-measure-table}
\end{table}

These examples show that  weight and magnitude are two independent components of the heterogeneity which have a great impact on the gain of average-$K$ over top-$K$ (adaptive gain).
\Cref{thm:regret-lower-bound} shows that this is true for all problems by providing a lower bound on the adaptive gain $\Delta_K$.
This lower bound is not tight in general, \ie the straddle strength value can be much smaller than the final value of the adaptive gain (\eg see Example~2 in \Cref{tab:heterogeneity-measure-table}).
However, this result and its analysis provide insights into which problems would benefit most from the average-$K$ strategy.

\begin{proposition}
  \label{thm:regret-lower-bound}
  
  For $K < C$, where $C$ is the number of classes, we have that
  \begin{equation*}
    \Delta_K \geq d_{K,1} = \E[X,X']{ \left( \tilde{\eta}_{K+1}(X) - \tilde{\eta}_K(X') \right)^+ },
  \end{equation*}
  and, more tightly, for $K \leq \frac{C}{2}$, 
  \begin{equation*}
    \Delta_K \geq \sum_{k=1}^K d_{K,k} = \sum_{k=1}^K \E[X,X']{ \left( \tilde{\eta}_{K+k}(X) - \tilde{\eta}_{K+1-k}(X') \right)^+ }.
  \end{equation*}
\end{proposition}

\begin{proof}
  Using \Cref{eq:adaptive-gain}, we can rewrite the adaptive gain $\Delta_K$ as
  \begin{align*}
    \Delta_K
    &= \sum_k \E[X]{ \left| \eta_k(X) - \lambda_K \right| \ind{k \in \TopKSet_K^*(X) \SymDiff \AverageKSet_K^*(X)} } \\
    &= \sum_{k \leq K} \E[X]{ ( \lambda_K - \tilde{\eta}_k(X) ) \ind{\tilde{\eta}_k(X) < \lambda_K} } \\
    & \quad + \sum_{k > K} \E[X]{ ( \tilde{\eta}_k(X) - \lambda_K ) \ind{\tilde{\eta}_k(X) > \lambda_K} } \\
    &= \E[X]{ \sum_{k \leq K} ( \lambda_K - \tilde{\eta}_k(X) )^+ + \sum_{k > K} ( \tilde{\eta}_k(X) - \lambda_K )^+}
  \end{align*}
  where $(a)^+ = \max(a,0)$ is the positive part.

  From this, we can prove the first lower bound using
  \begin{align*}
    \Delta_K
    &\geq \E[X]{ ( \lambda_K - \tilde{\eta}_K(X) )^+ + ( \tilde{\eta}_{K+1}(X) - \lambda_K )^+ } \\
    &= \E[X,X']{ ( \lambda_K - \tilde{\eta}_K(X') )^+ + ( \tilde{\eta}_{K+1}(X) - \lambda_K )^+ } \\
    &\geq \E[X,X']{ ( \tilde{\eta}_{K+1}(X) - \tilde{\eta}_K(X') )^+ }.
  \end{align*}
  The last lower bound comes from the fact that, for any real numbers $a$, $b$ and $c$,
  \begin{equation*}
    (a-b)^+ + (b-c)^+ \geq \left( (a-b) + (b-c) \right)^+ = (a-c)^+.
  \end{equation*}
  
  For the second lower bound of the proposition, we use a tighter lower bound of the sum:
  \begin{align*}
    \Delta_K
    &= \E[X]{ \sum_{k \leq K} ( \lambda_K - \tilde{\eta}_k(X) )^+ + \sum_{k > K} ( \tilde{\eta}_k(X) - \lambda_K )^+} \\
    &\geq \E[X]{ \sum_{k=1}^K ( \lambda_K - \tilde{\eta}_k(X) )^+ + \sum_{k=K+1}^{2K} ( \tilde{\eta}_k(X) - \lambda_K )^+ } \\
    &= \sum_{k=1}^K \E[X,X']{ ( \lambda_K - \tilde{\eta}_{K+1-k}(X') )^+ + ( \tilde{\eta}_{K+k}(X) - \lambda_K )^+ } \\
    &\geq \sum_{k=1}^K \E[X,X']{ ( \tilde{\eta}_{K+k}(X) - \tilde{\eta}_{K+1-k}(X') )^+ }
  \end{align*}
  using the same lower bound as previously.
\end{proof}

\section{Consistent Estimation Procedures}
\label{sec:estimation-procedures}

In this section, we will study how to estimate the top-$K$ and average-$K$ predictors developed in \Cref{sec:preliminaries}.
The aim here is not to provide the best possible estimators but rather to give general principled estimation procedures which are consistent and that will allow us to carry out the experiments of \Cref{sec:experiments}.
In particular, the estimators we derive here are based on the optimal rules defined in \Cref{eq:optimal-top-k-classifier,eq:optimal-adaptive-top-k-classifier}.
To apply these rules, we need an estimator $\hat{\eta}$ of the conditional probability $\eta$ and directly plug it into the optimal rules.
Such estimators are called \emph{plug-in rules}.

In the rest of this section, we will be interested in the \emph{regret} of risks which is a way to quantify how sub-optimal are our estimators.
In particular, the \emph{top-$K$ regret} is defined as
\begin{equation*}
  \RegretFixed(\hat{\TopKSet}_K) \eqdef \Error(\hat{\TopKSet}_K) - \Error(\TopKSet_K^*),
\end{equation*}
where $\hat{\TopKSet}_K: \X \to \powerset{\Y}$ is an estimated top-$K$ classifier, \ie $\forall x, |\hat{\TopKSet}_K(x)| = K$, and $\TopKSet_K^*$ is the optimal top-$K$ classifier of \Cref{eq:optimal-top-k-classifier}.
Similarly, the \emph{average-$\cK$ regret} is defined as
\begin{equation*}
  \RegretAdapt(\hat{\AverageKSet}_\cK) \eqdef \Error(\hat{\AverageKSet}_\cK) - \Error(\AverageKSet_\cK^*),
\end{equation*}
\sloppy
where $\hat{\AverageKSet}_\cK: \X \to \powerset{\Y}$ is an estimated average-$\cK$ classifier, \ie ${\E[X]{| \hat{\AverageKSet}_\cK(X)| } = \cK}$, and $\AverageKSet_\cK^*$ is the optimal average-$K$ classifier of \Cref{eq:optimal-adaptive-top-k-classifier}.
In our case, $\hat{\TopKSet}_K$ and $\hat{\AverageKSet}_\cK$ are built from $\hat{\eta}$ and their exact definition are given in \Cref{sec:plugin-estimators}.

In general, it is hard to directly optimize these losses ($\Error(\hat{\TopKSet}_K)$ and $\Error(\hat{\AverageKSet}_\cK)$).
It is thus common to use a \emph{surrogate loss} $l$ which is designed to have nicer optimization properties (convexity for instance).
In this paper, we will focus on \emph{conditional probability estimation} (CPE) losses, that is to say losses of the form $l: \Y \times \ProbSimplex{C} \to \R^+$ where $\ProbSimplex{C}$ is the probability simplex of dimension $C$, \ie $\ProbSimplex{C} = \enscond{ p \in \R^C }{ \forall k, \, p_k \geq 0, \, \sum_k p_k = 1 }$.
In this case, its regret is defined as
\begin{equation*}
  \regret_l(\hat{\eta}) \eqdef \E[X,Y]{ l(Y, \hat{\eta}(X)) } - \inf_{\hat{\eta}'} \E[X,Y]{ l(Y, \hat{\eta}'(X)) }.
\end{equation*}
One important question when using surrogate losses is to understand how optimizing this loss will eventually minimize the losses that really interest us, \ie top-$K$ and average-$\cK$ errors ($\Error(\hat{\TopKSet}_K)$ and $\Error(\hat{\AverageKSet}_\cK)$).
In particular, there are two simple properties that one would want to be satisfied in general by such losses: \emph{calibration} and \emph{consistency} \citep{Bartlett2006,Tewari2007}.

\begin{definition}[Calibration]
  A loss is said to be \emph{calibrated} for a task if it satisfies:
  \begin{equation*}
    \regret_l(\hat{\eta}) = 0 \quad \Rightarrow \quad \regret_{\mathrm{task}}(\hat{S}) = 0,
  \end{equation*}
  \ie if the minimizers of the surrogate losses also minimize the original loss.
\end{definition}

\begin{definition}[Consistency]
  A loss is said to be \emph{consistent} if, given a sequence of estimators $(\hat{\eta}_n)$ from which we build a sequence of set classifiers $(\hat{S}_n)$, it satisfies:
  \begin{equation*}
    \regret_l(\hat{\eta}_n) \to 0 \quad \Rightarrow \quad \regret_{\mathrm{task}}(\hat{S}_n) \to 0,
  \end{equation*}
  \ie if, given a sequence of estimators converging to the minimizer of the surrogate loss, that sequence of estimators will lead to the construction of optimal set classifiers.
\end{definition}

Typically, this sequence of estimators is indexed by the training set size.
Ideally, we want to have an estimator which regret $\regret_l(\hat{\eta}_n)$ goes to zero when the sample size goes to infinity.
This consistency property will guarantee that, if we are able to build such an estimator $\hat{\eta}_n$, the resulting set classifier built from it will converge towards the optimal set classifier.

These two properties are, in general, not equivalent.
Indeed, although they are known to be equivalent in the binary classification case \citep{Bartlett2006}, they may not be equivalent  in the multi-class case \citep{Tewari2007}.
There is thus no reason to believe these properties are equivalent in the problems that interest us here.
For top-$K$ classification, \citet{Lapin2016,Lapin2017} proposed some losses which are calibrated for top-$K$ classification, but \citet{Yang2020} showed  that some of their losses were not consistent.
For average-$K$ classification, \citet{Denis2017} showed that losses of a certain form are indeed consistent but their result is not applicable to the negative log-likelihood loss, which is by far the most widely used loss in classification with neural networks.

The main aim of this section is to provide consistency results for a special class of losses (strongly proper losses) which includes several widely used losses in practice and in particular the negative log-likelihood.
The section is organized as follows.

We first define the plug-in estimators for top-$K$ and average-$K$ estimation and show that they are consistent given consistent estimators of the conditional probability $\eta(x)$.
Then, we study \emph{proper losses} and a subclass of such losses called \emph{strongly proper losses} which provides a way to build such estimators of $\eta(x)$, we also show that several widely used losses---such as negative log-likelihood---actually have such a property.
Note that these results on strongly proper losses, presented in \Cref{sec:strongly-proper-losses}, are novel and of independent interest.
In \Cref{sec:experiments}, we will discuss some good practices and show how these results apply in practice when training neural networks.

\subsection{Plug-in Estimators and Their Regret Bounds}
\label{sec:plugin-estimators}

The plug-in rule for the top-$K$ classifier is given by
\begin{equation}
  \label{eq:top-k-plugin-estimator}
  \hat{\TopKSet}_K(x) \eqdef \enscond{ \hat{\sigma}_x(k) }{ k \in \ens{1,\dots,K} }
\end{equation}
where $\hat{\sigma}_x$ is a permutation of $\ens{1,\dots,C}$ such that $\hat{\eta}_{\hat{\sigma}_x(1)}(x) \geq \hat{\eta}_{\hat{\sigma}_x(2)} (x) \geq \ldots \geq \hat{\eta}_{\hat{\sigma}_x(C)}(x)$.
For the adaptive strategy, we need to apply a thresholding resulting in the following plug-in estimator:
\begin{equation}
  \label{eq:adaptive-top-k-plugin-estimator}
  \hat{\AverageKSet}_\cK (x) \eqdef \hat{\AverageKSet}_\cK^+(x) \cup \widetilde{\hat{\AverageKSet}}_\cK^{=}(x),
\end{equation}
where
\begin{equation*}
  \hat{\AverageKSet}_\cK^+(x) = \enscond{ k \in \Y }{ \hat{\eta}_k(x) > G_{\hat{\eta}}^{-1}(\cK) } ,
\end{equation*}
and $\widetilde{\hat{\AverageKSet}}_\cK^{=}$ is any deterministic classifier predicting a subset of the labels produced by $\hat{\AverageKSet}_\cK^{=}$ defined as
\begin{equation*}
  \hat{\AverageKSet}_\cK^{=}(x) = \enscond{ k \in \Y }{ \hat{\eta}_k(x) = G_{\hat{\eta}}^{-1}(\cK) } ,
\end{equation*}
such that $\Info(\widetilde{\hat{\AverageKSet}}_\cK^{=}) = \cK - \Info(\hat{\AverageKSet}_\cK^+)$.
Such a deterministic estimated classifier exists under the same conditions as the existence of a deterministic optimal classifier $\widetilde{\AverageKSet}_\cK^{=}$, \eg under the conditions of \Cref{thm:existence-deterministic-classifier}.
Here, $G_{\hat{\eta}}$ and its generalized inverse are defined similarly to $G_\eta$ and $G_\eta^{-1}$ (see \Crefparenthesis{eq:G-function,eq:G-inverse-function}):
\begin{align*}
  G_{\hat{\eta}}(\lambda) &\eqdef \sum_k \pr[X]{ \hat{\eta}_k(X) > \lambda }, \\
  G_{\hat{\eta}}^{-1}(\cK) &\eqdef \min \enscond{ \lambda \in [0,1] }{ G_{\hat{\eta}}(\lambda) \leq \cK } .
\end{align*}

Note that here the threshold $G_{\hat{\eta}}^{-1}(\cK)$ is not the same as for the optimal rule, \ie in general $G_{\hat{\eta}}^{-1}(\cK) \neq G_{\eta}^{-1}(\cK)$.
Indeed, as we want to preserve the constraint on the average set size, \ie that $\Info(\hat{\AverageKSet}_\cK) = \Info(\AverageKSet_\cK^*) = \cK$, then we need to adapt the threshold on $\hat{\eta}$.

The following two theorems give upper bounds on the regret of top-$K$ and average-$K$ errors when using the plug-in estimators previously defined.
These bounds are expressed in terms of the estimation error of $\hat{\eta}$ as measured by $\E[X]{ \| \eta(X) - \hat{\eta}(X) \|_1 }$ and thus show that improving the estimation of $\eta$ by $\hat{\eta}$ has the effect of improving  the  top-$K$ and average-$K$ classifiers.

\begin{theorem}
  \label{thm:plugin-regret-bounds-top-k}

  Let $K \in \{1, \dots,C \}$ be a fixed set size.
  Given an estimator $\hat{\eta}$ of $\eta$, the associated top-$K$ plug-in estimator $\hat{\TopKSet}_K$ as defined in \Cref{eq:top-k-plugin-estimator} has the following regret bound,
  \begin{equation*}
    \RegretFixed(\hat{\TopKSet}_K) \leq \E[X]{ \| \eta(X) - \hat{\eta}(X) \|_1 }.
  \end{equation*}
\end{theorem}

\begin{theorem}
  \label{thm:plugin-regret-bounds-adaptive-top-k}
  
  Let $\cK \in [0,C]$ be a fixed average set size.
  Given an estimator $\hat{\eta}$ of $\eta$, the associated average-$\cK$ plug-in estimator $\hat{\AverageKSet}_\cK$ as defined in \Cref{eq:adaptive-top-k-plugin-estimator} has the following regret bound,
  \begin{equation*}
    \RegretAdapt(\hat{\AverageKSet}_\cK) \leq \E[X]{ \| \eta(X) - \hat{\eta}(X) \|_1 }.
  \end{equation*}
\end{theorem}

\begin{proof}[Proof of \Cref{thm:plugin-regret-bounds-top-k}]
  For readability, as we will do the proof for a fixed $X$, we omit $X$ in the formulas.
  Using the fact that $\sum_{k \in \TopKSet_K^*} \hat{\eta}_k - \sum_{k \in \hat{\TopKSet}_K} \hat{\eta}_k \leq 0$ as a direct consequence of the definition of $\hat{\TopKSet}_K$, the point-wise regret can be upper bounded using
  \begin{align*}
    \RegretFixed(\hat{\TopKSet}_K; X)
    &= \sum_{k \in \TopKSet_K^*} \eta_k - \sum_{k \in \hat{\TopKSet}_K} \eta_k \\
    &= \sum_{k \in \TopKSet_K^*} (\eta_k - \hat{\eta}_k) + \sum_{k \in \hat{\TopKSet}_K} (\hat{\eta}_k - \eta_k) + \sum_{k \in \TopKSet_K^*} \hat{\eta}_k - \sum_{k \in \hat{\TopKSet}_K} \hat{\eta}_k \\
    &\leq \sum_{k \in \TopKSet_K^*} (\eta_k - \hat{\eta}_k) + \sum_{k \in \hat{\TopKSet}_K} (\hat{\eta}_k - \eta_k) \\
    &= \sum_{k \in \TopKSet_K^* \setminus \hat{\TopKSet}_K} (\eta_k - \hat{\eta}_k) + \sum_{k \in \hat{\TopKSet}_K \setminus \TopKSet_K^*} (\hat{\eta}_k - \eta_k) \\
    &\leq \sum_{k \in \TopKSet_K^* \SymDiff \hat{\TopKSet}_K} | \eta_k - \hat{\eta}_k | \\
    &\leq \| \eta - \hat{\eta} \|_1.
  \end{align*}
  Taking the expectation over $X$ concludes the proof.
\end{proof}

\begin{proof}[Proof of \Cref{thm:plugin-regret-bounds-adaptive-top-k}]
  For readability, we omit $X$ in the formulas.
  Using \Cref{thm:adaptive-set-prediction-excess-risk}, we have
  \begin{equation*}
    \RegretAdapt(\hat{\AverageKSet}_\cK)
    = \E[X]{ \sum_k | \eta_k - \lambda | \ind{k \in \AverageKSet_\cK^* \SymDiff \hat{\AverageKSet}_\cK} }
  \end{equation*}
  If $k \in \AverageKSet_\cK^* \setminus \hat{\AverageKSet}_\cK$, we have $\eta_k \geq \lambda$ and $\hat{\eta}_k \leq \hat{\lambda}$, thus,
  \begin{equation*}
    | \eta_k - \lambda |
    = \eta_k - \lambda
    = (\eta_k - \hat{\eta}_k) + (\hat{\eta}_k - \hat{\lambda}) + (\hat{\lambda} - \lambda)
    \leq | \eta_k - \hat{\eta}_k | + ( \hat{\lambda} - \lambda ).
  \end{equation*}
  Similarly, if $k \in \hat{\AverageKSet}_\cK \setminus \AverageKSet_\cK^*$, we have $\eta_k \leq \lambda$ and $\hat{\eta}_k \geq \hat{\lambda}$, thus,
  \begin{equation*}
    | \eta_k - \lambda |
    = \lambda - \eta_k
    = (\lambda - \hat{\lambda}) + (\hat{\lambda} - \hat{\eta}_k) + (\hat{\eta}_k - \eta_k)
    \leq | \eta_k - \hat{\eta}_k | + ( \lambda - \hat{\lambda} ).
  \end{equation*}
  Finally, this gives,
  \begin{equation*}
    \RegretAdapt(\hat{\AverageKSet}_\cK)
    \leq \E[X]{ \| \eta - \hat{\eta} \|_1 + ( \hat{\lambda} - \lambda ) ( | \AverageKSet_\cK^* \setminus \hat{\AverageKSet}_\cK | - | \hat{\AverageKSet}_\cK \setminus \AverageKSet_\cK^* | ) }.
  \end{equation*}
  The second term is in fact equal to $( \hat{\lambda} - \lambda ) ( | \AverageKSet_\cK^* | - | \hat{\AverageKSet}_\cK | )$, taking the expectation over $X$ gives $( \hat{\lambda} - \lambda ) ( \Info(\AverageKSet_\cK^*) - \Info(\hat{\AverageKSet}_\cK) ) = 0$.
  Thus, taking the expectation over $X$ concludes the proof.
\end{proof}

Note that the previous theorems give upper bounds, so they do not imply that having bad estimators of $\eta$ will necessarily lead to bad estimators $\hat{\TopKSet}_K$ and $\hat{\AverageKSet}_{\cK}$.
However, these bounds are useful to derive the consistency results presented in the next section.

\subsection{Proper Losses and Strongly Proper Losses}
\label{sec:strongly-proper-losses}

Now that we have shown plug-in regret bounds, we study losses that allow the construction of estimators of $\eta$, namely proper losses.
We will focus on a subset of those losses that satisfy a stronger property: strongly proper losses.
The main result of this section, \Cref{thm:strongly-proper-losses-consistency}, shows that such losses are consistent for top-$K$ and average-$K$ classification.
We then apply this result to show that negative log-likelihood, Brier score, and others are examples of such losses and are thus also consistent for these tasks.
More generally, note that this subclass of proper losses have not been introduced in the multi-class setting and the results of this section are thus novel and of independent interest beyond set-valued classification.

Throughout this section, we will consider different quantities for a given $x \in \X$, but for simplicity, the dependence on $x$ will be omitted in the equations.

Given a loss $l: \Y \times \ProbSimplex{C} \to \R^+$, its conditional risk is defined as
\begin{equation*}
  L_l(\eta, \hat{\eta}) \eqdef \sum_k \eta_k l(k, \hat{\eta}).
\end{equation*}
We denote $L_l^*$ its infimum according to the second variable, \ie
\begin{equation*}
  L_l^*(\eta) \eqdef \inf_{\eta'} L_l(\eta, \eta').
\end{equation*}

We first start by recalling the definition of (strictly) proper loss \citep{Gneiting2007} which were studied and proved to be useful in the context of binary classification \citep{Reid2009} and of multi-class classification \citep{Vernet2011}.

\begin{definition}[(Strictly) proper loss]
  A loss $l: \Y \times \ProbSimplex{C} \to \R$ is said to be \emph{proper} if its conditional risk's infimum is attained by $\eta$, \ie
  \begin{equation*}
    L_l(\eta,\eta) = L_l^*(\eta).
  \end{equation*}
  If this infimum is uniquely attained by $\eta$, then the loss is said to be \emph{strictly proper}.
\end{definition}

The following definition is a generalization to the multi-class setting of the strongly proper losses which were introduced by \citet{Agarwal2014} for the binary case in the context of bipartite ranking.\footnote{To the best of our knowledge, this property of strongly proper losses in the case of multi-class classification was not published before.}

\begin{definition}[Strongly proper loss]
  \label{def:strongly-proper-loss}
  
  A loss $l: \Y \times \ProbSimplex{C} \to \R$ is said to be $\mu$-strongly proper if it satisfies
  \begin{equation*}
    L_l(\eta,\hat{\eta}) - L_l(\eta,\eta) \geq \frac{\mu}{2} \| \eta - \hat{\eta} \|_1^2,
  \end{equation*}
  with $\mu > 0$.
\end{definition}
It is easy to check that a strongly proper loss is necessarily strictly proper.
Indeed, if $\hat{\eta} \neq \eta$, $L_l(\eta,\hat{\eta}) > L_l(\eta,\eta)$, and thus $L_l(\eta,\eta) = L_l^*(\eta)$.
Therefore, we can rewrite the previous definition as a lower bound on the point-wise regret of $l$:
\begin{equation*}
  \regret_l(\hat{\eta}; x) \geq \frac{\mu}{2} \| \eta(x) - \hat{\eta}(x) \|_1^2.
\end{equation*}
The converse is not true in general: a strictly proper loss might not be strongly proper.

Note that, when applying this definition in the binary case, $\| \eta - \hat{\eta} \|_1$ becomes $2 | \eta - \hat{\eta} |$ resulting in a multiplicative factor of $4$ for the parameter $\mu$ compared to the definition of \citet{Agarwal2014}.
Note also that, because of norm equivalence in the finite dimension, the actual norm used in the definition is not a restriction: one can easily vary the norm by changing the parameter $\mu$ as needed.
We will illustrate this with examples of such losses later in this section.

The choice of the $\ell_1$-norm is justified by the following proposition, which
 shows how the estimation error of $\hat{\eta}$ is controlled by the regret of a strongly proper loss.

\begin{proposition}
  \label{thm:cond-prob-est-error-upper-bound}
  
  Let $l$ be a $\mu$-strongly proper loss.
  Then, we have
  \begin{equation*}
    \E[X]{ \| \eta(X) - \hat{\eta}(X) \|_1 } \leq \sqrt{ \frac{2}{\mu} \regret_l(\hat{\eta}) }
  \end{equation*}
\end{proposition}

\begin{proof}
  Using the convexity of the function $x \mapsto x^2$ and the definition of strongly proper losses, we have
  \begin{align*}
    \E[X]{ \| \eta(X) - \hat{\eta}(X) \|_1 }
    &\leq \sqrt{ \E[X]{ \| \eta(X) - \hat{\eta}(X) \|_1^2  } } \\
    &\leq \sqrt{ \frac{2}{\mu} \E[X]{ L_l(\eta(X),\hat{\eta}(X)) - L_l^*(\eta(X)) } } \\
    &= \sqrt{ \frac{2}{\mu} \regret_l(\hat{\eta}) }.
  \end{align*}
\end{proof}

This proposition allows us to  prove the consistency of strongly proper losses with respect to the set-valued classification problems we consider in this paper.

\begin{theorem}
  \label{thm:strongly-proper-losses-consistency}
  Any strongly proper loss is consistent for top-$K$ and average-$K$ classification.
\end{theorem}

\begin{proof}
  This is a direct consequence of the bounds of \Cref{thm:plugin-regret-bounds-top-k,thm:plugin-regret-bounds-adaptive-top-k} and of \Cref{thm:cond-prob-est-error-upper-bound}.
  Considering a sequence of estimators $(\hat{\eta}_n)$ such that $\regret_l(\hat{\eta}_n)$ tends towards zero will force the regret $\regret_{\mathrm{task}}(\hat{S}_n)$ to also tend to zero.
\end{proof}

This result tells us that, if we are able to minimize the risk of a strongly proper loss as the size of the training set increases, then we can use the resulting estimator of $\hat{\eta}$ to build a top-$K$ and an average-$K$ classifier from it which will converge to the related optimal set classifiers.
We now give a few examples of such strongly proper losses.

\begin{example}[Negative log-likelihood]
  The negative log-likelihood loss (NLL) defined as
  \begin{equation*}
    l_{\log}(k, \hat{\eta}) = - \log \hat{\eta}_k
  \end{equation*}
  is $1$-strongly proper.\footnote{See the previous comments for why, in the definition of \citet{Agarwal2014} in the binary case, it is $4$-strongly proper.}
  It is thus consistent for the set-valued classification problems we consider in this paper.

  This is a direct consequence of Pinsker's inequality \citep[Theorem 4.19]{Boucheron2013} which states that
  \begin{equation*}
    \KLdiv(\eta, \hat{\eta}) \geq 2 \| \eta - \hat{\eta} \|_{TV}^2 = \frac{1}{2} \| \eta - \hat{\eta} \|_1^2
  \end{equation*}
  where $\KLdiv$ is the Kullback-Leibler divergence and $\| . \|_{TV}$ is the total variation distance.
  Simply noting that the point-wise regret is equal to
  \begin{equation*}
    \regret_{\log}(\hat{\eta}; x) = \KLdiv(\eta(x), \hat{\eta}(x)),
  \end{equation*}
  concludes the proof by \Cref{def:strongly-proper-loss}.
\end{example}

\begin{example}[Brier score]
  Let the square loss, a.k.a. Brier score, defined as
  \begin{equation*}
    l_{\mathrm{sq}}(k, \hat{\eta}) = \frac{1}{2} \sum_{k'} (\ind{k'=k} - \hat{\eta}_{k'})^2.
  \end{equation*}
  For this loss, the point-wise regret is equal to
  \begin{equation*}
    \regret_{\mathrm{sq}}(\hat{\eta};x)
    = \frac{1}{2} \| \eta(x) - \hat{\eta}(x) \|_2^2
    \geq \frac{1}{2C}  \| \eta(x) - \hat{\eta}(x) \|_1^2.
  \end{equation*}
  It is thus $\frac{1}{C}$-strongly proper and consistent for the set-valued classification problems we consider in this paper.
\end{example}

\begin{example}[One-versus-all strongly proper losses]
  Let $l_b: \{-1,+1\} \times [0,1] \to \R$ be a binary $\mu$-strongly proper loss \citep[as defined by][]{Agarwal2014} and define the following one-versus-all loss
  \begin{equation*}
    l_{\mathrm{ova}}(k, \hat{\eta}) = l_b(+1, \hat{\eta}_k) + \sum_{k' \neq k} l_b(-1, \hat{\eta}_{k'}).
  \end{equation*}
  We have
  \begin{align*}
    L_{\mathrm{ova}}(\eta, \hat{\eta})
    &= \sum_k \eta_k l_b(+1, \hat{\eta}_k) + \sum_{k, k' \neq k} \eta_k l_b(-1, \hat{\eta}_{k'}) \\
    &= \sum_k \eta_k l_b(+1, \hat{\eta}_k) + \sum_{k, k' \neq k} \eta_{k'} l_b(-1, \hat{\eta}_k) \\
    &= \sum_k \eta_k l_b(+1, \hat{\eta}_k) + \sum_k (1-\eta_k) l_b(-1, \hat{\eta}_k) \\
    &= \sum_k L_b(\eta_k, \hat{\eta}_k).
  \end{align*}
  It is easy to check that
  \begin{align*}
    \regret_{\mathrm{ova}}(\hat{\eta}; x)
    &= \sum_k \regret_b(\hat{\eta}_k; x) \\
    &\geq \sum_k \frac{\mu}{2} \big( \eta_k(x) - \hat{\eta}_k(x) \big)^2 \\
    &= \frac{\mu}{2} \| \eta(x) - \hat{\eta}(x) \|_2^2 \\
    &\geq \frac{\mu}{2C} \| \eta(x) - \hat{\eta}(x) \|_1^2.
  \end{align*}
  Thus, losses built from univariate $\mu$-strongly proper losses are $\frac{\mu}{C}$-strongly proper and thus consistent.
\end{example}

Now that we have shown the consistency of the proposed estimation procedures, we can experiment with them on real-world data sets.

\section{Experiments on Real Data Sets}
\label{sec:experiments}

In this section, we carry out experiments on image classification data sets to see how well they accord with the previous theoretical results.
First, we detail how top-$K$ and average-$K$ classifiers are effectively estimated using a finite data set.
Next, we study the proposed estimation procedures empirically to see how they behave in practice.
Then, in order to test the potential benefit of average-$K$, we carry out controlled experiments showing how different forms of noise impact the relative performance of top-$K$ classification and its adaptive counterpart.
We then compare top-$K$ and average-$K$  on \dataset{ImageNet} and on fine-grained classification data sets to determine whether real-world data sets can benefit from average-$K$ classification.
We have chosen to restrict ourselves to image data sets for these experiments because their ambiguity is visually evident.
However, top-$K$ and average-$K$ can be applied to any type of data.

\subsection{Finite Sample Estimation of Top-$K$ and Average-$K$ Classifiers From $\hat{\eta}$}

Here, we assume that we are given an estimator $\hat{\eta}$ of the conditional class probability $\eta$.
This estimator is learned on the training set.

Deriving the top-$K$ classifier is direct and we will always predict sets of size $K$ on the test set.
On the other hand, for the average-$K$ classifier, we need to find the appropriate threshold $G_{\hat{\eta}}^{-1}(K)$.
In practice, we have only a finite number of sample to estimate this threshold.
There are two ways to do so: (i) use a calibration set distinct from the test set, or, (ii) directly find the threshold on the test set.
In the first case, the average size of the sets predicted on the test set will, in general, be different from $K$.
In the second case, we can guarantee that the average set size will be equal to $K$ on the test set although it might not be exactly the theoretical threshold $G_{\hat{\eta}}^{-1}(K)$.
We will be using the second option as we want to compare average-$K$ and top-$K$ classifiers for the same (average) set size $K$.

\begin{algorithm}[t]
  \caption{Computing the average-$K$ classifier on a finite data set 
  }
  \label{alg:average-K-error-rate-computation}
  \begin{algorithmic}[1]
    \Require{
      \Statex $y_{\mathrm{scores}}$ is the predicted scores matrix of size $N \times C$.
      \Statex $K$ is an integer representing the average set size constraint.
      \Statex \Call{Flatten}{$y_{\mathrm{scores}}$} returns the flattened array of matrix $y_{\mathrm{scores}}$ (\ie the matrix collapsed into one dimension row-by-row).
      \Statex \Call{Quantile}{$s, q$} returns the first $q$-quantile (with lower interpolation) of the values contained into the one-dimensional array $s$, can be implemented with complexity $O(N \times C)$.
    }
    \Statex
    \Function{Threshold}{$y_{\mathrm{scores}}, K$} \Comment{$O(N \times C)$ complexity}
      \State $s \gets $ \Call{Flatten}{$y_{\mathrm{scores}}$}
      \State $q \gets 1 - K / C$
      \State \Return{\Call{Quantile}{$s, q$}}
    \EndFunction
    \Statex
    \Function{AverageKSets}{$y_{\mathrm{scores}}, K$} \Comment{$O(N \times C)$ complexity}
      \State $\lambda \gets $ \Call{Threshold}{$y_{\mathrm{scores}}, K$}
      \State $S^+ \gets y_{\mathrm{scores}} > \lambda$
      \State $S^= \gets y_{\mathrm{scores}} = \lambda$
      \State $K_{missing} \gets K - $ \Call{Sum}{$S^+$} $ / N$ \Comment{$K_{missing} \times N$ is thus an integer}
      \State $\widetilde{S}^= \gets $ take the first $K_{missing} \times N$ labels in $S^=$
       \State \Return{$S^+ \cup \widetilde{S}^=$}
    \EndFunction
  \end{algorithmic}
\end{algorithm}

\sloppy
In the case where the test set $\cD_{\mathrm{test}}$ is finite of size $N$, \ie ${\cD_{\mathrm{test}} = ( x_1, \dots, x_N )}$, finding the appropriate threshold simply consists in finding the correct quantile in the distribution of the predicted scores $\hat{\eta}(x_1), \dots, \hat{\eta}(x_N)$.
This can be done efficiently with complexity $O(N \times C)$ where $C$ is the number of classes.
We can then use this threshold $\lambda$ to compute the sets $S^+$ and $S^=$ for all test points $i \in \{ 1, \dots, N \}$ defined as
\begin{align*}
  S^+(x_i) &= \enscond{ k \in \Y }{ \hat{\eta}_k(x_i) > \lambda } , \\
  S^=(x_i) &= \enscond{ k \in \Y }{ \hat{\eta}_k(x_i) = \lambda } .
\end{align*}
\sloppy
If the average set size $S^+$ on $\cD_{\mathrm{test}}$ is not equal to $K$, \ie if ${\frac{1}{N} \sum_{i=1}^N |S^+(x_i)| < K}$, we can complete it by breaking the ties contained in $S^=(x_i)$.
This can be done by arbitrarily ordering these ties using the indices of the classes and of the samples and  keeping only the ${N \times K - \sum_{i=1}^N |S^+(x_i)|}$ first ones.
At the end, we obtain sets of average size $K$ as expected as:
\begin{equation*}
  \frac{1}{N} \sum_{i=1}^N |S^+(x_i)| + \frac{1}{N} \left( N \times K - \sum_{i=1}^N |S^+(x_i)| \right) = K .
\end{equation*}
The corresponding complete algorithm is given in \Cref{alg:average-K-error-rate-computation}.
A Python version is also provided in \Cref{sec:python-code}.
This algorithm was used for all the experiments in this paper.

\subsection{Practical Estimation of Top-$K$ and Average-$K$ Classifiers}

From a theoretical standpoint, when the conditional probabilities of the classes given the input are exactly known, optimal average-$K$ classifiers are always at least as good as optimal top-$K$ ones.
In this subsection, we study what happens in practice when we  have only estimators of such quantities resulting in classifiers which may be sub-optimal.

As we are interested in the estimation procedures independently of a specific value of $K$, we will regularly measure the \emph{mean top-$K$ error} and the \emph{mean average-$K$ error}, where the mean is taken over all integer values of $K$:
\begin{align*}
  \text{mean top-$K$ error} &= \frac{1}{C} \sum_{K=1}^C \Error(\hat{\TopKSet}_K), \\
  \text{mean average-$K$ error} &= \frac{1}{C} \sum_{K=1}^C \Error(\hat{\AverageKSet}_K),
\end{align*}
where $\hat{\TopKSet}_K$ and $\hat{\AverageKSet}_K$ are the estimators defined in \Cref{eq:top-k-plugin-estimator,eq:adaptive-top-k-plugin-estimator}.

Throughout this section, we will perform  experiments on the \dataset{CIFAR-10} data set \citep{Krizhevsky2009},
which is a standard image classification data set containing 60,000 color images of size $32 \times 32$ illustrating 10 classes.
Although the classes are sufficiently distinct to render the original task quite simple, there is appreciable ambiguity, as shown in \Cref{fig:cifar10h-human-ambiguity-examples}, because the image resolution is low.
It is thus a good data set to perform these experiments.

\subsubsection{Estimation Using Neural Networks}

Using neural networks has become the standard approach in image classification.
However, while efficient in practice, the training of such models is not fully understood from a theoretical point-of-view.
In order to see how the proposed estimation procedures for top-$K$ and average-$K$ behave with such models, we perform standard training  on \dataset{CIFAR-10} to see what happens.

\begin{figure}[t]
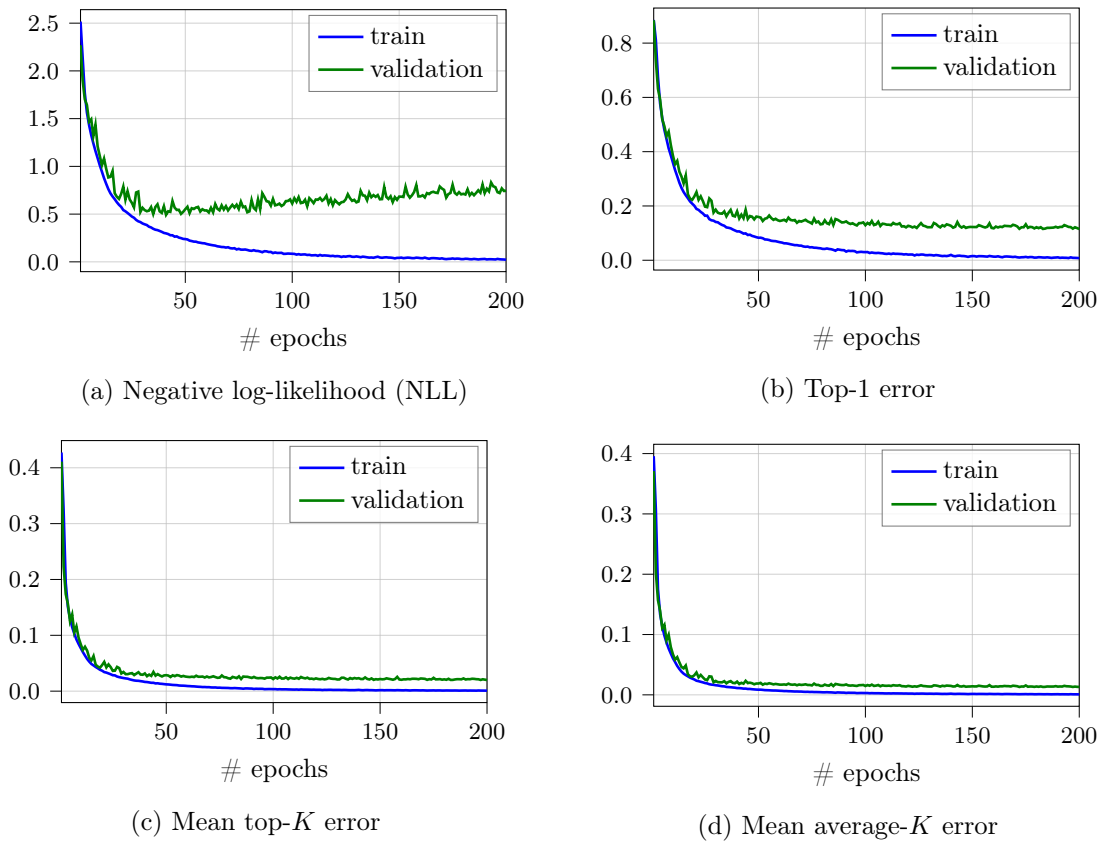

  \hfill
  \begin{subfigure}{.475\textwidth}
    \centering
    \input{plots/experiments/estimation_procedures/learning_curves_cifar10_loss.pgf}%
    \caption{Negative log-likelihood (NLL)}
    \label{fig:learning-curves-nll}
  \end{subfigure}
  \hfill
  \begin{subfigure}{.475\textwidth}
    \centering
    \input{plots/experiments/estimation_procedures/learning_curves_cifar10_top1_error.pgf}%
    \caption{Top-1 error}
    \label{fig:learning-curves-top-1-error}
  \end{subfigure}
  \hfill
  \\
  \vspace{.3cm}
  \\
  \hfill
  \begin{subfigure}{.475\textwidth}
    \centering
    \input{plots/experiments/estimation_procedures/learning_curves_cifar10_mean_top_k_error.pgf}%
    \caption{Mean top-$K$ error}
    \label{fig:learning-curves-mean-top-K-error}
  \end{subfigure}
  \hfill
  \begin{subfigure}{.475\textwidth}
    \centering
    \input{plots/experiments/estimation_procedures/learning_curves_cifar10_mean_average_k_error.pgf}%
    \caption{Mean average-$K$ error}
    \label{fig:learning-curves-mean-avg-K-error}
  \end{subfigure}
  \hfill
  \caption{
    Evolution of the different metrics through the training process of a neural network on \dataset{CIFAR-10}.
    Training optimizes the negative log-likelihood (NLL) on the training set.
  }
  \label{fig:learning-curves}
\end{figure}

First, we train a ResNet-44 \citep{He2016} using the standard training procedure detailed in the referenced paper.
In \Cref{fig:learning-curves}, the learning curves for different metrics are shown.
It is known \citep{Guo2017} that, although the negative log-likelihood (NLL) measured on the validation can increase through training as shown in \Cref{fig:learning-curves-nll}, the top-1 error continues to decrease as shown in \Cref{fig:learning-curves-top-1-error}.
Interestingly, and perhaps more surprisingly, the same comment can be made for the estimators of top-$K$ and average-$K$ as can be seen respectively in \Cref{fig:learning-curves-mean-top-K-error} and \Cref{fig:learning-curves-mean-avg-K-error}.

This shows that, in spite of the worst-case analysis in \Cref{sec:estimation-procedures}, carefully monitoring NLL on the validation set is in fact not necessary to obtain good estimators of top-$K$ and average-$K$ classifiers.
In the end,  training strategies similar to those used to minimize errors for top-1 classification work well. 
This means that virtually nothing needs to be changed in the training. Even a pre-trained model can  be used directly for top-$K$ and average-$K$ prediction.

Moreover, this empirical finding suggests that the usual training procedures for neural networks, although not necessarily learning a probability-calibrated model, still learn useful scores for average-$K$.
In turn, this suggests that the relative ordering of the conditional probability for each class and each sample is often preserved in the learned scores.
In \Cref{app:complementary-experiments-learning-curves}, complementary experiments on other data sets are provided, they show that these findings are generalizable to other data sets.

\begin{figure}[t]
  \centering
  \input{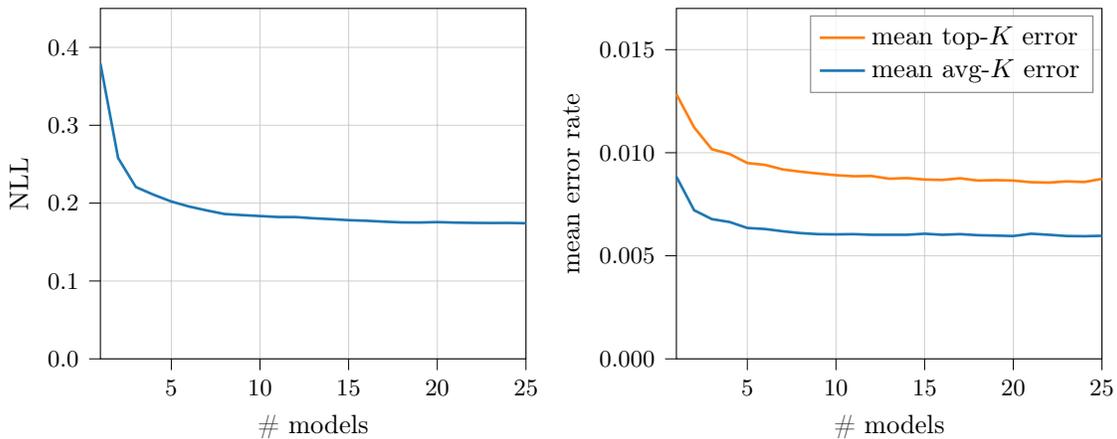}%
  \caption{
    Impact of the number of models present in the ensemble (\dataset{CIFAR-10}).
    More machine learning models reduce the error.
  }
  \label{fig:ensemble-impact}
\end{figure}

In order to further improve the learned classifiers, we use ensembles of neural networks trained using different initialization and mini-batch sampling.
The impact of training with several models in the ensemble is shown in \Cref{fig:ensemble-impact}.
Both top-$K$ and average-$K$ benefit significantly from using ensembles.
Surprisingly, average-$K$ seems to require fewer models than top-$K$ to stabilize.
Here again, complementary experiments provided in \Cref{app:complementary-experiments-model-ensembling} show that these findings are also generalizable to other data sets.

\subsubsection{Impact of \textit{a posteriori} Probability Calibration}

In this section, we study the impact of probability calibration methods used once the training is completed.
These methods aim to improve the estimated probabilities outputted by an existing model \citep{Flach2017}.
They thus seem relevant for top-$K$ and average-$K$ classification.

We restrict our analysis to extensions to the multi-class case of Platt's scaling \citep{Platt1999} proposed by \citet{Guo2017} as these methods were shown to outperform  other approaches when used with neural networks.
We briefly recall them here.
They consist in taking an already trained model  and learning a logistic regression on the logits.
That is, if we denote by $z(x)$ the scores in the logits space predicted by a model for a given $x$, the original estimator of the conditional probability $\hat{\eta}_k$ is computed using
\begin{equation*}
  \hat{\eta}_k(x) = \softmax( z(x) )_k,
\end{equation*}
where $\softmax(z)_k = \frac{e^{z_k}}{\sum_{k'} e^{z_k'}}$.
This estimator is replaced by $\hat{\eta}'_k$ computed as
\begin{equation*}
  \hat{\eta}'_k(x) = \softmax( W z(x) + b )_k,
\end{equation*}
where $(W,b)$ is a new linear model fitted by minimizing the negative log-likelihood (NLL) on a calibration set.
We call this calibration \emph{full matrix scaling}.
If $W$ is a diagonal matrix, it is called \emph{vector scaling}.
Finally, it is called \emph{temperature scaling} if the calibration simply consists in learning a single temperature parameter $T$ and the resulting estimator is computed as
\begin{equation*}
  \hat{\eta}'_k(x) = \softmax( z(x) / T )_k.
\end{equation*}
Although extremely simple, this last method was shown by \citet{Guo2017} to be very efficient.

We compare the different calibration methods on the \dataset{CIFAR-10} data set by calibrating on the validation set a single previously trained model.
The results on the test set is shown in \Cref{tab:probability-calibration-methods} and \Cref{fig:probability-calibration-methods-error-curves}.
These values can be compared with those of ensembling shown in \Cref{fig:ensemble-impact}.
Although these calibration methods do reduce the negative log likelihood, which is what they were designed for, their impact on the top-$K$ and average-$K$ error rates is very small.

\begin{table}[t]
  \centering
  \begin{tabular}{lccc}
\toprule
\textbf{Calibration method} &    NLL &  mean top-$K$ error &  mean avg-$K$ error \\
\midrule
\textbf{No calibration     } & 0.3792 &              0.0128 &              0.0088 \\
\textbf{Temperature scaling} & 0.2444 &              0.0128 &              0.0087 \\
\textbf{Vector scaling     } & 0.2442 &              0.0129 &              0.0088 \\
\textbf{Full matrix        } & 0.2394 &              0.0126 &              0.0086 \\
\bottomrule
\end{tabular}

  \caption{
    The impact of the probability calibration methods on top-$K$ and average-$K$ metrics computed on \dataset{CIFAR-10} is small.
  }
  \label{tab:probability-calibration-methods}
\end{table}

\usetikzlibrary{backgrounds}

\begin{figure}[t]
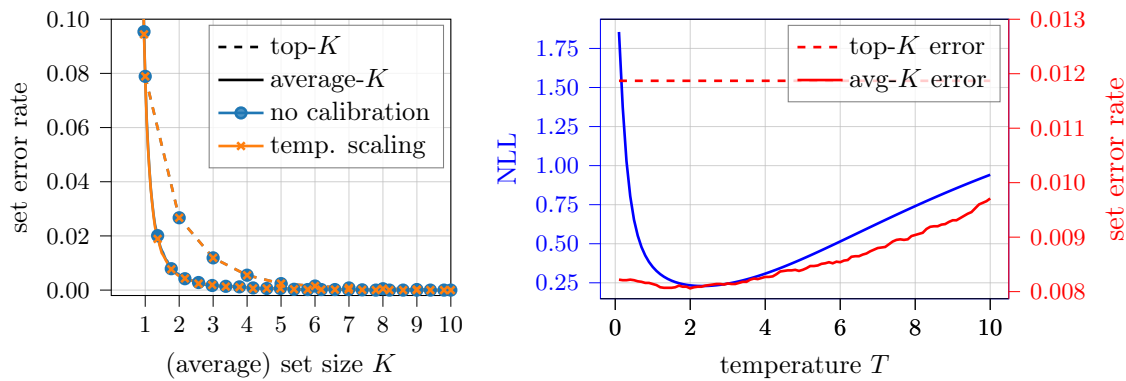

  \begin{subfigure}[t]{.4\textwidth}
    \input{plots/experiments/estimation_procedures/probability_calibration_methods_comparison.pgf}%
    \caption{
      Error rate vs. $K$ (set size) curve.
    }
    \label{fig:probability-calibration-methods-error-curves}
  \end{subfigure}
  \hfill
  \begin{subfigure}[t]{.575\textwidth}
    \input{plots/experiments/estimation_procedures/temperature_scaling_calibration_effect.pgf}%
    \caption{
      Effect of temperature scaling on NLL and top-$K$ and average-$K$ error rates.
    }
    \label{fig:temperature-scaling}
  \end{subfigure}
  \caption{
    Impact of probability calibration methods on top-$K$ and average-$K$ classifiers computed on \dataset{CIFAR-10}.
  }
  \label{fig:probability-calibration-methods}
\end{figure}

In \Cref{fig:temperature-scaling}, we study  the influence of temperature scaling on the model calibration in detail.
NLL is reduced by using a temperature of $T \approx 2.3$, which is greater than $1$, the default value without calibration.
The model was thus overconfident before the calibration: the scaling reduces the values of probabilities outputted by the model.
Temperature scaling has no influence on the top-$K$ error rate as it preserves the order of the classes for each sample taken individually.
By contrast, we can see that the temperature has an effect on the average-$K$ error but that error is still much lower than the top-$K$ error as long as $T$ is not too high.
Moreover, the standard temperature $T=1$ is very close to optimal.

These results suggest that, as noted in the previous section, the scores learned by a standard training of neural networks work well for top-$K$ and average-$K$.
Thus, additional probability calibration is not necessary. We will therefore not use these methods in the rest of the experiments.

\subsubsection{When Training Data Is Scarce, Average-$K$ Is Still Better}

\begin{figure}[t]
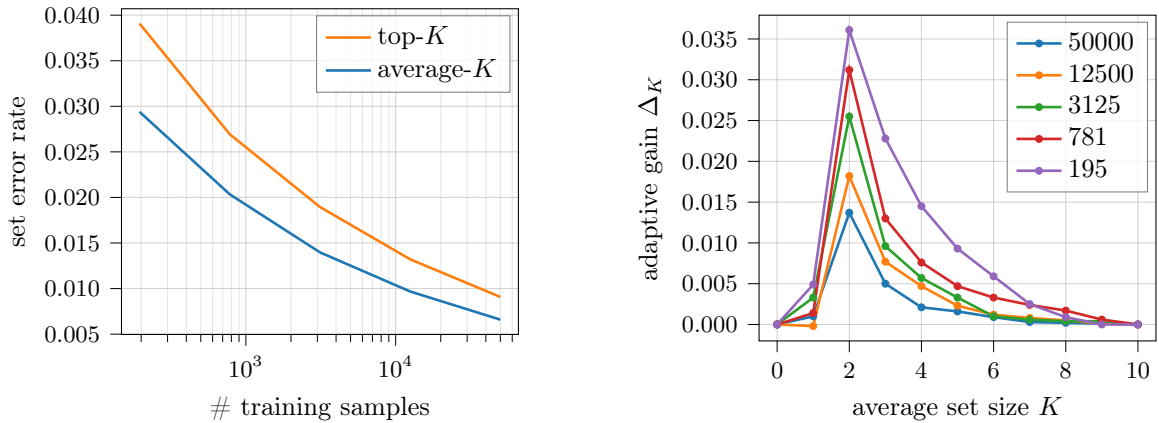

  \begin{subfigure}[t]{.45\textwidth}
    \input{plots/experiments/controlled_experiments/fashion_mnist_subsampled_mean_top_k_avg_k_error.pgf}%
    \caption{Mean error rate vs. number of training samples.}
  \end{subfigure}
  \hfill
  \begin{subfigure}[t]{.45\textwidth}
    \input{plots/experiments/controlled_experiments/fashion_mnist_subsampled_adaptive_gain_plot.pgf}%
    \caption{Adaptive gain with different  numbers of training samples.}
  \end{subfigure}
  \caption{
    Impact of training set size on top-$K$ and average-$K$ error rates measured on Fashion-MNIST's test set.
    The adaptive gain is highest when there are few samples.
  }
  \label{fig:training-set-size-impact}
\end{figure}

When the volume of training data is limited, one could expect top-$K$ classifiers to be easier to estimate than average-$K$ classifiers.
In problems for which optimal average-$K$ classifier would provide only a minor advantage over optimal top-$K$ classifier, we would expect the estimated top-$K$ to have a lower error rate than the estimated average-$K$.
Surprisingly, this expected phenomenon never occurred in our experiments. The estimated average-$K$ classifier was still better than the top-$K$ one.

To highlight this, we use the Fashion-MNIST data set \citep{Xiao2017}, which has a low degree of ambiguity.\footnote{MNIST has even less ambiguity but we did not use it for the experiments here because the task is too simple for neural networks which achieved very low error rate even with very few training samples.}
This data set consists of 60,000 training images and 10,000 test images of clothes containing 10 different classes.
We split the original training set into two parts: 50,000 for training purposes and 10,000 for validation of hyperparameters.
We then subsample this training set by dividing it by four for every experiment resulting in set sizes of $50000$, $12500$, $3125$, $781$, and $195$.
The results computed on the final test set are shown in \Cref{fig:training-set-size-impact}.
Not only is average-$K$ always better than top-$K$ but the gap  is wider for fewer samples.
Even when we have a limited training set, average-$K$ still improves upon top-$K$.

\subsection{Experiments on Altered Real Data: Potential of Average-$K$ Classification}

In this subsection, we consider experiments on image data sets in which we inject different types of noise to create ambiguity.
In particular, we focus on (i)~adding noise in the labels, and (ii)~downsampling the images.

\subsubsection{Label Noise Injection}

We first consider modifying the labels in order to inject noise.
We introduce a method to inject the noise in a principled way which will allow us to control its characteristics and to compute  the quantities highlighted in  previous sections.

To do so, we will start from a (nearly) unambiguous data set, namely, MNIST \citep{Lecun1998} for which the top-1 error rate is very low (less than $1\%$) even when using simple models.
Note moreover that this data set is balanced: every class has the same number of samples.
We then group together the classes such that, if a sample's label is in a certain group, we regenerate a new label by randomly sampling a label of that group.
This is equivalent to introducing a confusion matrix from which we regenerate the labels using $\prcond{ Y_\text{new} = k }{ Y_\text{old} = k' }$ with the values of the new confusion matrix.

\begin{figure}
  \centering
  \includegraphics[width=\textwidth]{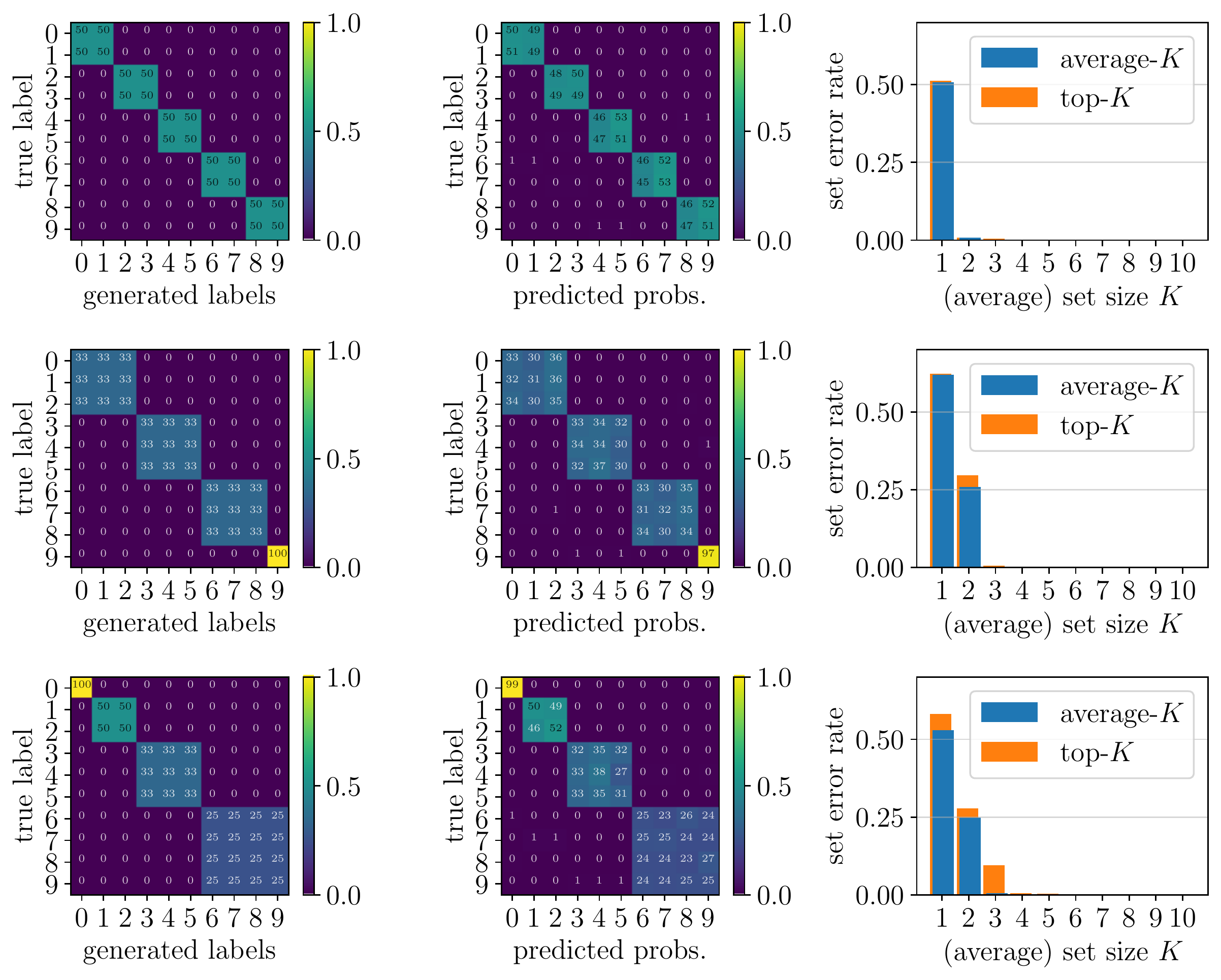}
  \caption{
    Influence of ambiguity injection via label noise.
    The first column shows the confusion matrix used to generate new labels: in the first  line, neighbor classes are fused in groups of two; in the second line, groups of size 3; in the last line,  the groups have different sizes.
    The first two lines show homogeneous  label noise injection while the last one shows heterogeneous label noise injection.
    The second column shows the predicted mean probabilities for samples of each class after the models are learned: the ambiguity is pretty well understood and estimated by the models.
    The last column shows the relative gain (orange) of average-$K$ over top-$K$ in each case.
    When the ambiguity is homogeneous, average-$K$ has little advantage over top-$K$.
    But when it is heterogeneous, the gain can be substantial.
  }
  \label{fig:mnist-label-noise-injection-experiments}
\end{figure}

The results appear in \Cref{fig:mnist-label-noise-injection-experiments}.
In the first  line, the classes are gathered into balanced groups of size 2 and in the second line balanced groups of size 3.
These cases are similar to the experiments carried out by \citet{Berrada2018} for top-$K$ classification.
As one can see, the ambiguity introduced here is very homogeneous and the average-$K$ classifier yields only a marginal gain over the top-$K$ classifier.

On the other hand, in the last line, we study injecting heterogeneous label noise by gathering classes into groups of different sizes.
As one can see, in this case, we can have a noticeable gap between top-$K$ and average-$K$ classifiers.
In particular, for $K = 3$, the error rate of the average-$3$ strategy is zero whereas the top-3 classifier has an error rate around 10\%.

These experiments show that, in general, when the confusion among classes is homogeneous, as measured by confusion matrices, top-$K$ is optimal or very close to average-$K$.
When the confusion is heterogeneous, average-$K$ yields a much lower error rate than top-$K$.
In any case, average-$K$ is always at least as good as top-$K$.
These empirical results are thus in line with the theoretical results of \Cref{sec:top-k-optimality}.

\subsubsection{Input Image Degradation Analysis}

We now consider experiments where noise is added to the input images.
In this case, it is less clear what type of or how much ambiguity we are introducing.
The aim of this section is to measure  what happens in this case.

Adding noise on MNIST did not introduce enough ambiguity, the task being too easy.
These experiments are thus carried out on Fashion-MNIST \citep{Xiao2017}.
Here, we use an ensemble of simple convolutional neural networks whose error rate on the original data set is around $6\%$.
The original ambiguity is thus higher than on MNIST but still not very high.
In order to inject ambiguity, we downsampled  all the images from all classes by the same amount. Some resulting sample images are shown in \Cref{fig:degraded-fashion-mnist-samples}.

\begin{figure}
  \centering
  \foreach \size in {28x28, 7x7, 4x4, 2x2} {
    \begin{subfigure}{.5\textwidth}
      \includegraphics[width=\textwidth]{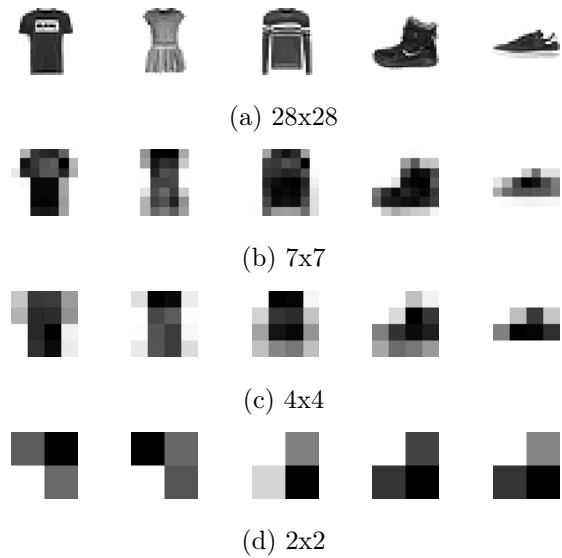}
      \caption{\size}
    \end{subfigure}
    ~\\
  }
  \vspace{-.15cm}

  \caption{
    Fashion-MNIST with various amounts of downsampling.
    Each column corresponds to a single image taken from five different classes.
    From left to right: T-shirt, dress, pull-over, ankle boot and sneaker.
  }
  \label{fig:degraded-fashion-mnist-samples}
\end{figure}

As expected, the greater the noise, the more ambiguous the images as shown in the global increase of error rates displayed in \Cref{fig:error-rate-avg-set-size-curves-degraded-fashion-mnist}.
One might expect that the introduced ambiguity is homogeneous because the noise is applied uniformly on all the images.
However, \Cref{fig:regret-degraded-fashion-mnist} shows that the gap between top-$K$ and average-$K$ strategies is quite wide.
This implies that the ambiguity is in fact heterogeneous.
As can be seen for the samples from \Cref{fig:degraded-fashion-mnist-samples}, some classes are actually still recognizable even though the noise level is high, \eg T-shirts (first column) and sneakers (last column), still are not confused.
On the other hand, some classes become harder to distinguish, such as different classes of shoes, \eg boots and sneakers in the two last columns.

Thus, in general, adding noise to the input will render the task heterogeneously ambiguous.
This experiment increases the benefits of average-$K$ in the presence of input noise and with increased ambiguity.
\Cref{fig:error-rate-avg-set-size-curves-degraded-fashion-mnist} shows how much average-$K$ can reduce the error rate.

\begin{figure}
  \begin{subfigure}[b]{.5\textwidth}
    \centering
    \includegraphics[width=\textwidth]{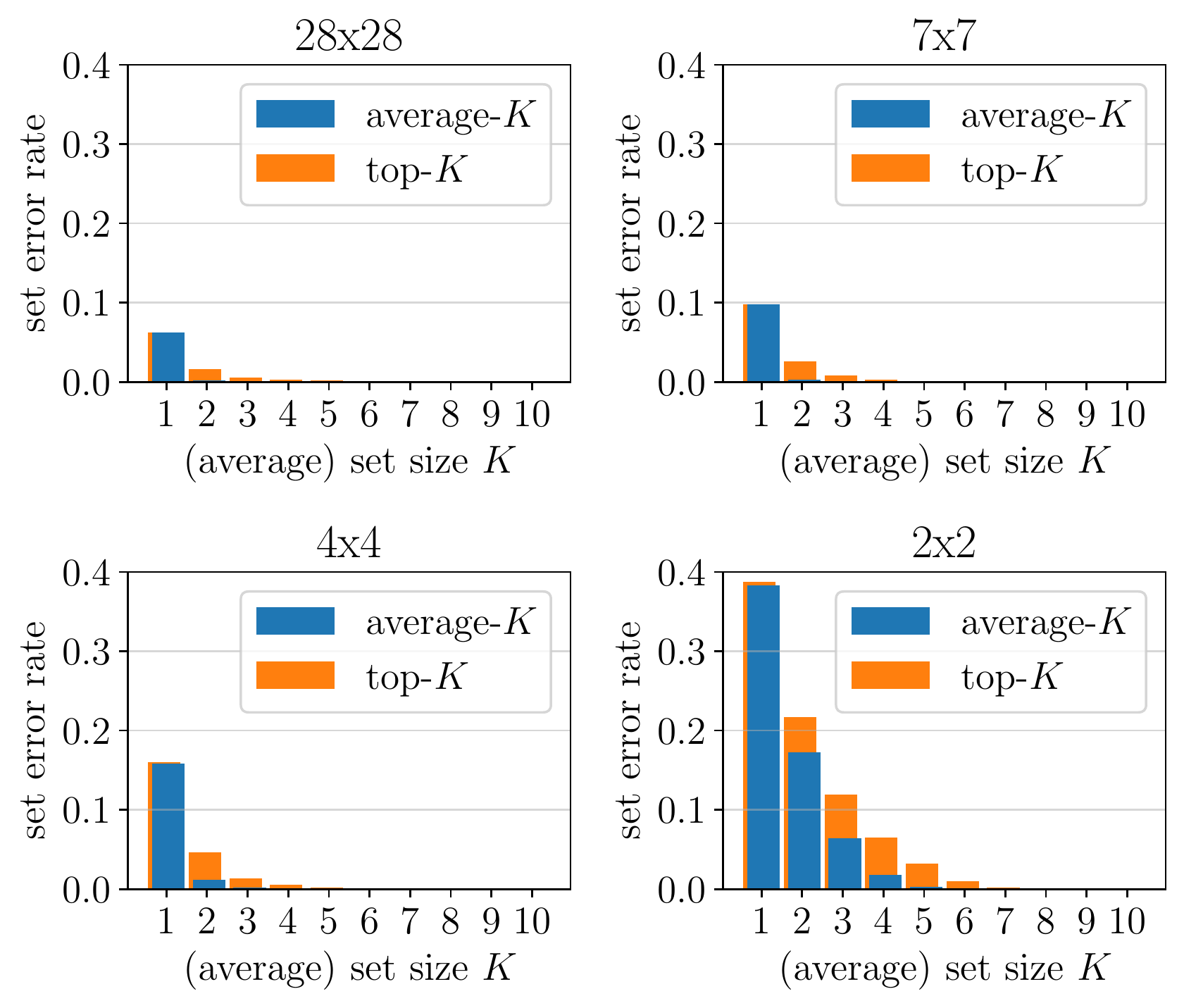}
    \caption{
      Error rate / average set size curves.
    }
    \label{fig:error-rate-avg-set-size-curves-degraded-fashion-mnist}
  \end{subfigure}
  \hfill
  \begin{subfigure}[b]{.45\textwidth}
    \centering
    \input{plots/experiments/controlled_experiments/fashion_mnist_downsampled_adaptive_gain_plot.pgf}
    \caption{
      Adaptive gain curves.
    }
    \label{fig:regret-degraded-fashion-mnist}
  \end{subfigure}
  
  \caption{
    Impact of image downsampling in Fashion-MNIST on top-$K$ and average-$K$ classifiers.
    The greater the ambiguity, the larger the adaptive gain.
  }
\end{figure}

\subsection{Experiments on Unaltered Real-world Data Sets}

We now perform experiments on real data sets in order to see if the previous comments on the usefulness of average-$K$ apply to practical tasks.

\subsubsection{ImageNet}

\dataset{ImageNet} \citep{Russakovsky2015} is a natural choice of data set for our demonstration.
Indeed, as explained in the previously cited article, \dataset{ImageNet} is known to have several objects per image as it is used for different learning tasks including classification, object detection and localization.
The class retained for an image is not necessarily the biggest object in the scene nor the most numerous.
It only has to appear in the scene.
We refer the reader to \citet{Russakovsky2015} for more details on the way the data set was collected and annotated.
For these reasons, in the classification task, the main metric considered is top-5 error rate.
However, it is known that images do not contain exactly the same number of objects: \dataset{ImageNet} is thus a natural candidate for average-$K$.\footnote{Note that the choice of $K=5$ is not discussed in the previous paper, in particular, the average number of objects per image is not given and is likely to be different from 5.}

\begin{figure}[t]
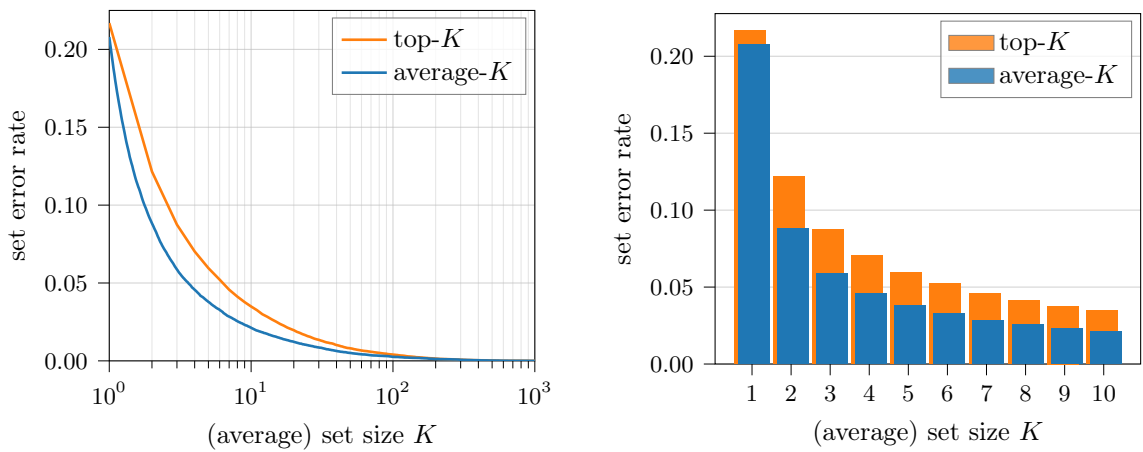

  \begin{subfigure}{.475\textwidth}
    \centering
    \input{plots/experiments/real_world_datasets/imagenet_error_rate_average_set_size_plot.pgf}%
  \end{subfigure}
  \hfill
  \begin{subfigure}{.475\textwidth}
    \centering
    \input{plots/experiments/real_world_datasets/imagenet_error_rate_reduction_bar_plot.pgf}%
  \end{subfigure}
  \caption{
    \dataset{ImageNet} top-$K$ and average-$K$ error rates for all values of $K$ computed on the official validation set.
    \dataset{ImageNet}'s main metric is the  top-5 error rate.
    For $K=5$,  average-5 reduces the error rate by ($\approx 37\%$) compared with top-$5$, as indicated by the additional error shown in orange.
  }
  \label{fig:imagenet-set-error-rates}
\end{figure}

\begin{figure}[t]
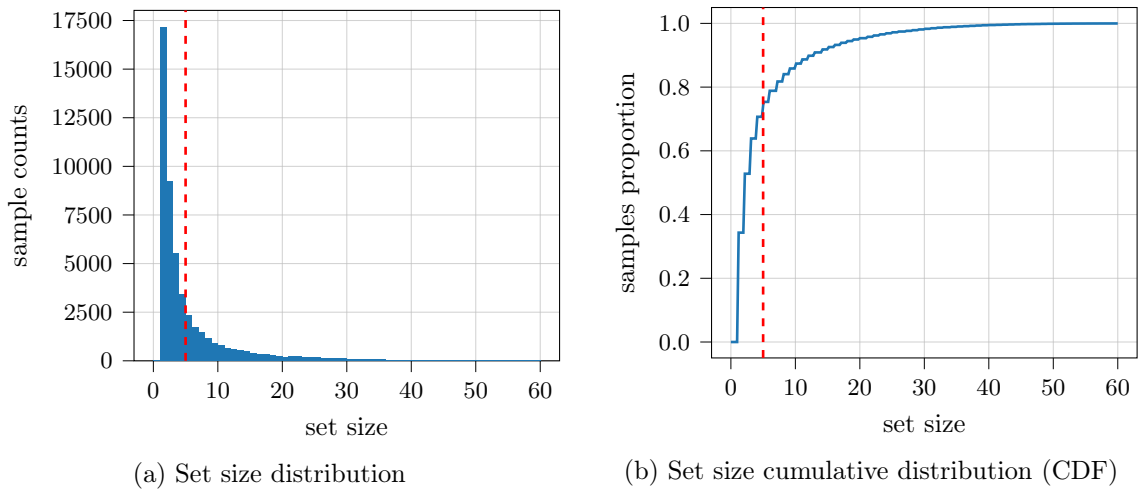

  \begin{subfigure}{.475\textwidth}
    \centering
    \input{plots/experiments/real_world_datasets/imagenet_set_sizes_distribution.pgf}%
    \caption{Set size distribution}
  \end{subfigure}
  \hfill
  \begin{subfigure}{.475\textwidth}
    \centering
    \input{plots/experiments/real_world_datasets/imagenet_set_sizes_distribution_cdf.pgf}%
    \caption{Set size cumulative distribution (CDF)}
  \end{subfigure}
  \caption{
    \dataset{ImageNet} set size distribution for the estimated adaptive top-5 classifier (on the validation set).
    As shown, the distribution of set sizes for average-$5$ does not peak at 5 (red dashed vertical line) but varies over a wide range of different set sizes.
    For most images, the classifier requires only a small set of classes  but some samples need a large set.
  }
  \label{fig:imagenet-set-sizes-distribution}
\end{figure}

For this experiment, we used an off-the-shelf already trained ResNet-152 model from PyTorch model zoo \citep{Paszke2019}.
The results on the official validation set (the test set is private) are shown in \Cref{fig:imagenet-set-error-rates}.
Average-$K$ greatly reduces the error rate compared to top-$K$.
In particular, for $K=5$, the official performance metric, the relative reduction is  37\%.
Average-$K$ is thus particularly well suited for this scenario.
To further analyze the difference between the two predictors on this data set, the distribution of the sizes of sets predicted by average-$K$ are given in \Cref{fig:imagenet-set-sizes-distribution}.
Interestingly, the distribution does not have a mode around 5 but is strictly decreasing.
Most of the distribution is concentrated at values below 5: 75\% of the sets have a size lower than 5.
Note that the distribution has a rather long tail: $5\%$ of the sets have a size larger than 20 and the maximum set size is 78.

\subsubsection{Fine-grained Visual Classification Data Sets}

\begin{figure}
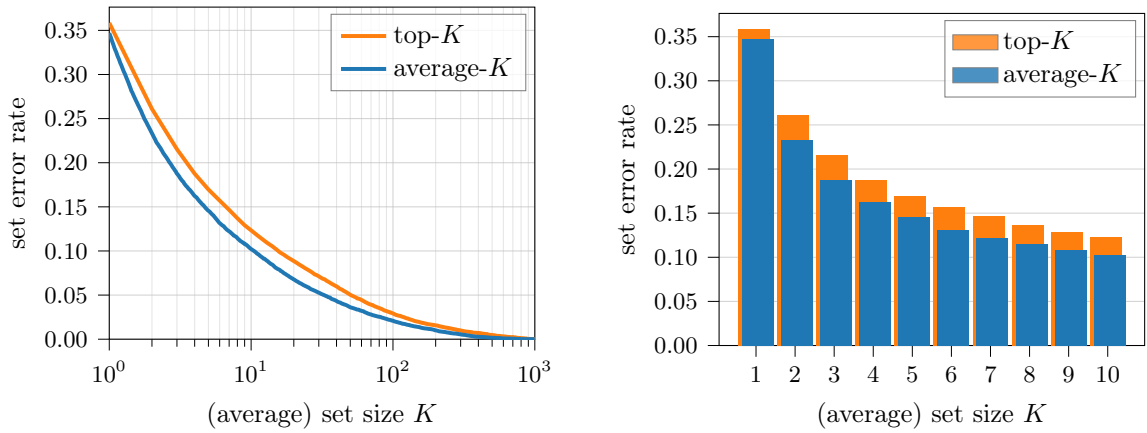
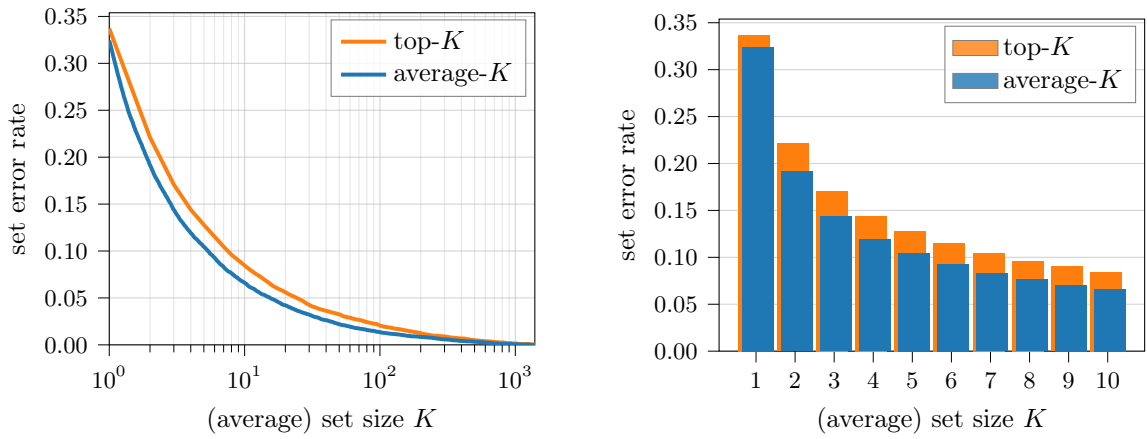
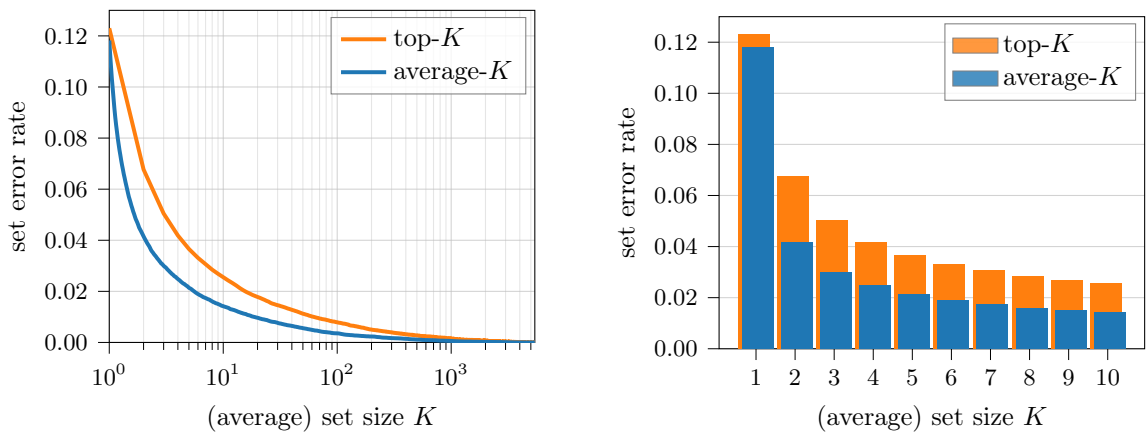

  \centering

  \begin{subfigure}{\textwidth}
    \centering
    \input{plots/experiments/real_world_datasets/plantclef2015_error_rate_average_set_size_plot.pgf}%
    \hfill
    \input{plots/experiments/real_world_datasets/plantclef2015_error_rate_reduction_bar_plot.pgf}%
    \caption{PlantCLEF2015}
  \end{subfigure}
  
  \vspace{.4cm}
  
  \begin{subfigure}{\textwidth}
    \centering
    \input{plots/experiments/real_world_datasets/fgvc5_fungi_error_rate_average_set_size_plot.pgf}%
    \hfill
    \input{plots/experiments/real_world_datasets/fgvc5_fungi_error_rate_reduction_bar_plot.pgf}%
    \caption{FGVC5-Fungi}
  \end{subfigure}
  
  \vspace{.4cm}
  
  \begin{subfigure}{\textwidth}
    \centering
    \input{plots/experiments/real_world_datasets/fgvc6_butterfliesmoths_error_rate_average_set_size_plot.pgf}%
    \hfill
    \input{plots/experiments/real_world_datasets/fgvc6_butterfliesmoths_error_rate_reduction_bar_plot.pgf}%
    \caption{FGVC6-ButterfliesMoths}
  \end{subfigure}

  \caption{
    For all fine-grained visual classification data sets, top-$K$ suffers a greater error rate than average-$K$, as shown in orange.
  }
  \label{fig:fgvc-datasets-error-rate-results}
\end{figure}

In this section, we study the performance of the average-$K$ strategy on fine-grained visual classification (FGVC) tasks.
While top-$K$ is a classic metric in this context, such tasks are likely to benefit more from average-$K$.
Unlike \dataset{ImageNet} where the ambiguity arises from the fact that there are several objects in the image, in FGVC data sets, usually, each image has only a single object  but the classes can resemble one another closely, so it can be hard to distinguish them from the given image.
We carry out experiments on three large-scale life science data sets:
\begin{itemize}
\item \dataset{PlantCLEF2015}: plant species recognition \citep{Goeau2015},
\item \dataset{FGVC5-Fungi}: mushroom species recognition,\footnote{\url{https://sites.google.com/view/fgvc5/competitions/fgvcx/fungi}} and
\item \dataset{FGVC6-ButterfliesMoths}: butterfly and moth species recognition.\footnote{\url{https://sites.google.com/view/fgvc6/competitions/butterflies-moths-2019}}
\end{itemize}
\Cref{tab:fgvc-datasets-characteristics} shows the characteristics of these data sets.
In particular, they require being able to discriminate among many different species (from 1000 to over 5000).
Note that the official metric of both \dataset{FGVC5-Fungi} and \dataset{FGVC6-ButterfliesMoths} is the top-3 error rate.

\begin{table}[t]
  \centering
  \begin{tabular}{lcc}
  \toprule
  \textbf{Name} & \textbf{\# samples} & \textbf{\# classes} \\
  \midrule
  PlantCLEF2015 & 113,205 & 1,000 \\
  FGVC5-Fungi & 85,578 & 1,394 \\
  FGVC6-ButterfliesMoths & 473,438 & 5,419 \\
  \bottomrule
\end{tabular}

  \caption{Fine-grained visual classification (FGVC) data set characteristics.}
  \label{tab:fgvc-datasets-characteristics}
\end{table}

The results are displayed in \Cref{fig:fgvc-datasets-error-rate-results}.
The relative improvement of average-$K$ depends on the data set.
In particular, the gain is high for \dataset{FGVC6-ButterfliesMoths}, yielding around a $41\%$ error rate reduction for $K=3$, while being lower but still noticeable for \dataset{PlantCLEF2015} and \dataset{FGVC5-Fungi}, both around $16\%$ for $K=3$.
Although the improvement is dependent on the data set, in all cases, using average-$K$ reduces  the error rate compared to top-$K$, sometimes appreciably.

\section{Conclusion and Future Work}

In this paper, we have studied the possible benefits of adaptively constructing sets of classes in classification problems with ambiguity. The classic approach, top-$K$, always generates the top $K$ most likely classes. Our adaptive generalization,  average-$K$, constructs sets of various sizes depending on the degree of ambiguity of the input, but keeping an overall average budget of $K$ classes per input.

At a theoretical level, we   characterized and quantified the nature of ambiguity in a data set that would cause average-$K$ to enjoy a lower error rate than top-$K$ in the infinite sample regime.
We then provided consistent procedures to estimate top-$K$ and average-$K$ classifiers.
Empirically, we  carried out experiments on real-world data sets to test the usefulness of average-$K$ in practice.  Our experiments show that there is adaptive gain on both real and synthetic data sets, even when there is little training data.

We have focused our theoretical analysis on the infinite sample regime in order to study how average-$K$ classifiers handle task ambiguity.
The natural next theoretical step would be to extend this analysis to the finite sample regime to understand how such set-valued classifiers behave in the presence of model uncertainty.
Indeed, even in the absence of task ambiguity, in which case average-$K$ classification is asymptotically useless, adaptive set-valued classification approaches might still  lower the error rate when training data is limited.

Another related research direction is to derive specialized losses which  explicitly target average-$K$ classification.
Such losses might lead to  better average-$K$ classifiers when there is little data.
In particular, it would be interesting to see if the gaps observed between average-$K$ and top-$K$ in the experiments of \Cref{sec:experiments} could be broadened using such loss measures.

Finally, we have focused on a specific formulation of set-valued classification.
However, other formulations have been proposed in the literature (see \Cref{sec:related-work}).
Extending this analysis to them would allow a better  understanding of when to use which set-valued classification framework.

\acks{%
  The authors are grateful to the OPAL infrastructure from Université Côte d'Azur for providing resources and support.
  DS has been partly supported as an INRIA International Chair, NYU WIRELESS, U.S. National Science Foundation grants 1934388, 1840761, and 1339362.
}

\newpage
\appendix

\section{Python Code for Average-$K$ Classifier Computation}
\label{sec:python-code}

The following listing provides a Python implementation of \Cref{alg:average-K-error-rate-computation} used for the experiments.
The predicted sets are represented as a two-dimensional boolean array of size $N \times C$ where $N$ is the number of samples and $C$ is the number of classes.

\begin{minted}{python}
import numpy as np

def find_threshold(avg_set_size, y_score):
    n_classes = y_score.shape[-1]

    q = 1 - (avg_set_size / n_classes)
    t = np.quantile(y_score, q=q, interpolation="lower")

    return t

def average_k_classifier(y_score, avg_set_size):
    n_samples = y_score.shape[0]

    threshold = find_threshold(avg_set_size, y_score)
    avg_k_set = (y_score > threshold)
    ties_set = (y_score == threshold)

    missing_labels = n_samples * avg_set_size - np.sum(avg_k_set)
    i, j = np.where(ties_set)
    i, j = i[:missing_labels], j[:missing_labels]

    avg_k_set[i, j] = True

    return avg_k_set
\end{minted}

\section{Complementary Experiments}
\label{app:complementary-experiments}

In this section, we provide complementary experiments to the ones of \Cref{sec:experiments}.
They essentially consist of similar experiments but carried out on different data sets to highlight the generalization of the conclusion drawn in the main text to other data sets.

\subsection{Training Curves}
\label{app:complementary-experiments-learning-curves}

\Cref{fig:learning-curves-plantclef2015,fig:learning-curves-fgvc5_fungi,fig:learning-curves-fgvc6_butterflies_moths} show the training curves (similarly to \Cref{fig:learning-curves}) for the models trained respectively on \dataset{PlantCLEF2015}, \dataset{FGVC5-Fungi} and \dataset{FGVC6-ButterfliesMoths}.
Note that, in these cases, a decay of the learning rate is performed midway during training.

\newcommand{\plotlearningcurves}[2]{
  \begin{figure}[p]
    \hfill
    \begin{subfigure}{.475\textwidth}
      \centering
      \input{plots/experiments/estimation_procedures/learning_curves_#1_loss.pgf}%
      \caption{Negative log-likelihood (NLL)}
    \end{subfigure}
    \hfill
    \begin{subfigure}{.475\textwidth}
      \centering
      \input{plots/experiments/estimation_procedures/learning_curves_#1_top1_error.pgf}%
      \caption{Top-1 error}
    \end{subfigure}
    \hfill
    \\
    \vspace{.3cm}
    \\
    \hfill
    \begin{subfigure}{.475\textwidth}
      \centering
      \input{plots/experiments/estimation_procedures/learning_curves_#1_mean_top_k_error.pgf}%
      \caption{Mean top-$K$ error}
    \end{subfigure}
    \hfill
    \begin{subfigure}{.475\textwidth}
      \centering
      \input{plots/experiments/estimation_procedures/learning_curves_#1_mean_average_k_error.pgf}%
      \caption{Mean average-$K$ error}
    \end{subfigure}
    \hfill
    \caption{#2}
    \label{fig:learning-curves-#1}
  \end{figure}
}

  \begin{figure}[p]
    \hfill
    \begin{subfigure}{.475\textwidth}
      \centering
      \input{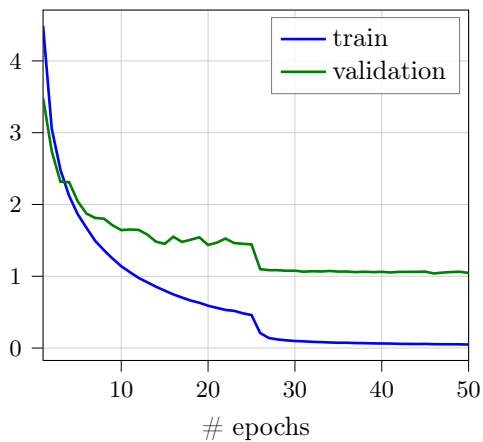}%
      \caption{Negative log-likelihood (NLL)}
    \end{subfigure}
    \hfill
    \begin{subfigure}{.475\textwidth}
      \centering
      \input{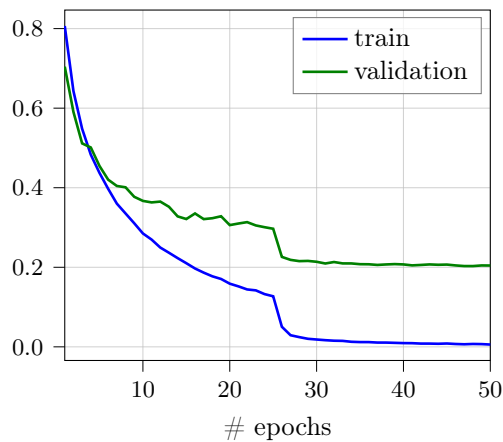}%
      \caption{Top-1 error}
    \end{subfigure}
    \hfill
    \\
    \vspace{.3cm}
    \\
    \hfill
    \begin{subfigure}{.475\textwidth}
      \centering
      \input{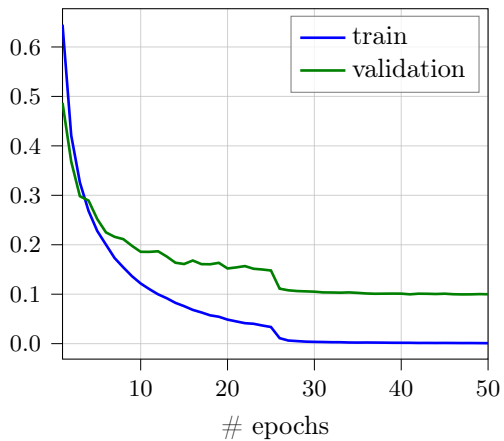}%
      \caption{Mean top-$K$ error}
    \end{subfigure}
    \hfill
    \begin{subfigure}{.475\textwidth}
      \centering
      \input{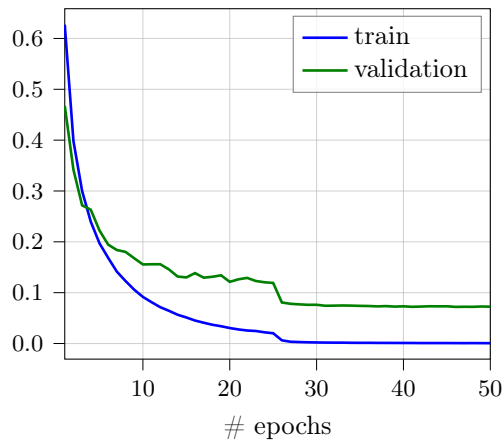}%
      \caption{Mean average-$K$ error}
    \end{subfigure}
    \hfill
    \caption{
  Evolution of the different metrics through the training process of a neural network on \dataset{PlantCLEF2015}.
  Mean top-$K$ and mean average-$K$ are calculated based on $K$ values from 1 to 10 (calculated uniformly for each value of $K$).
  During training, we fine-tune the pre-trained neural network using negative log-likelihood (NLL) on the training set.
  This shows that standard transfer learning methods work well for both top-$K$ and average-$K$.
}
    \label{fig:learning-curves-plantclef2015}
  \end{figure}

  \begin{figure}[p]
    \hfill
    \begin{subfigure}{.475\textwidth}
      \centering
      \input{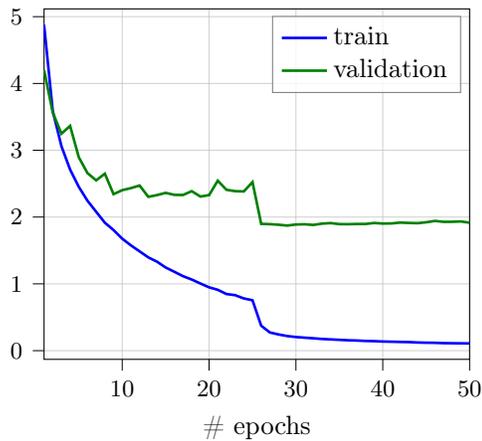}%
      \caption{Negative log-likelihood (NLL)}
    \end{subfigure}
    \hfill
    \begin{subfigure}{.475\textwidth}
      \centering
      \input{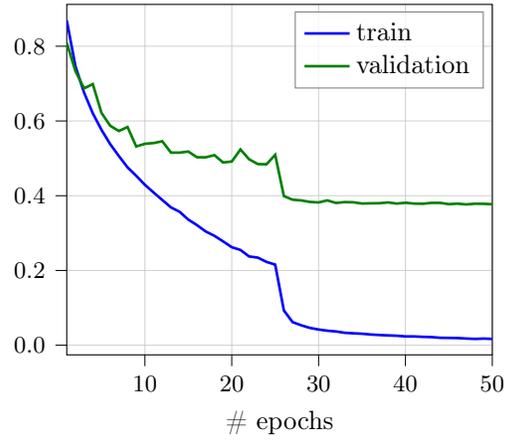}%
      \caption{Top-1 error}
    \end{subfigure}
    \hfill
    \\
    \vspace{.3cm}
    \\
    \hfill
    \begin{subfigure}{.475\textwidth}
      \centering
      \input{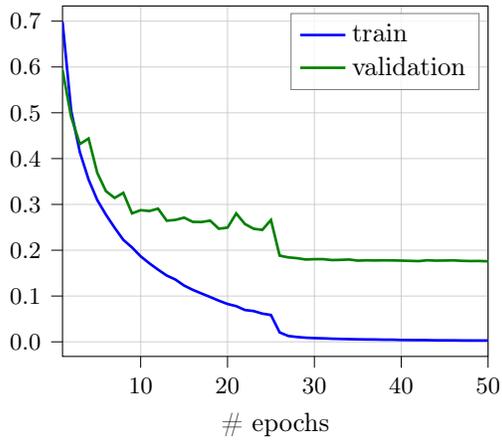}%
      \caption{Mean top-$K$ error}
    \end{subfigure}
    \hfill
    \begin{subfigure}{.475\textwidth}
      \centering
      \input{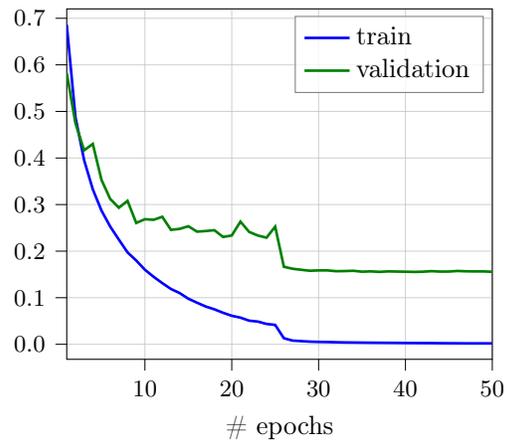}%
      \caption{Mean average-$K$ error}
    \end{subfigure}
    \hfill
    \caption{
  Evolution of the different metrics through the training process of a neural network on \dataset{FGVC5-Fungi}.
  Mean top-$K$ and mean average-$K$ are calculated based on $K$ values from 1 to 10 (calculated uniformly for each value of $K$).
  During training, as for \dataset{PlantCLEF2015}, we fine-tune the pre-trained neural network using negative log-likelihood (NLL) on the training set.
  This is further evidence that standard transfer learning methods work well for both top-$K$ and average-$K$.
}
    \label{fig:learning-curves-fgvc5_fungi}
  \end{figure}

  \begin{figure}[p]
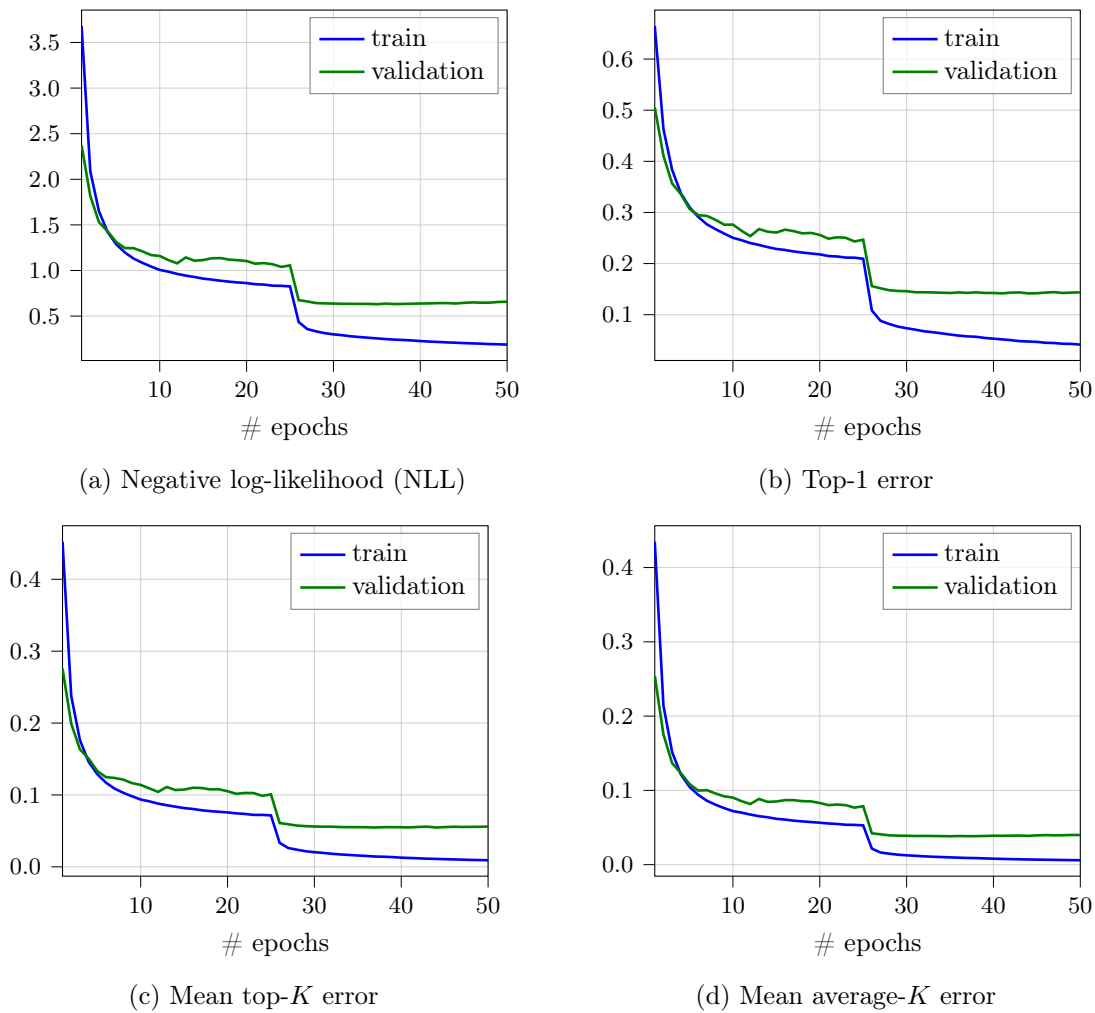

    \hfill
    \begin{subfigure}{.475\textwidth}
      \centering
      \input{plots/experiments/estimation_procedures/learning_curves_fgvc6_butterflies_moths_loss.pgf}%
      \caption{Negative log-likelihood (NLL)}
    \end{subfigure}
    \hfill
    \begin{subfigure}{.475\textwidth}
      \centering
      \input{plots/experiments/estimation_procedures/learning_curves_fgvc6_butterflies_moths_top1_error.pgf}%
      \caption{Top-1 error}
    \end{subfigure}
    \hfill
    \\
    \vspace{.3cm}
    \\
    \hfill
    \begin{subfigure}{.475\textwidth}
      \centering
      \input{plots/experiments/estimation_procedures/learning_curves_fgvc6_butterflies_moths_mean_top_k_error.pgf}%
      \caption{Mean top-$K$ error}
    \end{subfigure}
    \hfill
    \begin{subfigure}{.475\textwidth}
      \centering
      \input{plots/experiments/estimation_procedures/learning_curves_fgvc6_butterflies_moths_mean_average_k_error.pgf}%
      \caption{Mean average-$K$ error}
    \end{subfigure}
    \hfill
    \caption{
  Evolution of the different metrics through the training process of a neural network on \dataset{FGVC6-ButterfliesMoths}.
  Mean top-$K$ and mean average-$K$ are calculated based on $K$ values from 1 to 10 (calculated uniformly for each value of $K$).
  During training, as for \dataset{PlantCLEF2015}, we fine-tune the pre-trained neural network using negative log-likelihood (NLL) on the training set.
  Once again, standard transfer learning methods work well for both top-$K$ and average-$K$.
}
    \label{fig:learning-curves-fgvc6_butterflies_moths}
  \end{figure}

\newpage
\subsection{Impact of Model Ensembling}
\label{app:complementary-experiments-model-ensembling}

\Cref{fig:ensemble-impact-appendix} shows the impact of model ensembling (similarly to \Cref{fig:ensemble-impact}) for models trained respectively on \dataset{PlantCLEF2015}, \dataset{FGVC5-Fungi} and \dataset{FGVC6-ButterfliesMoths}.

\begin{figure}[h]
  \centering
  \begin{subfigure}{\textwidth}
    \input{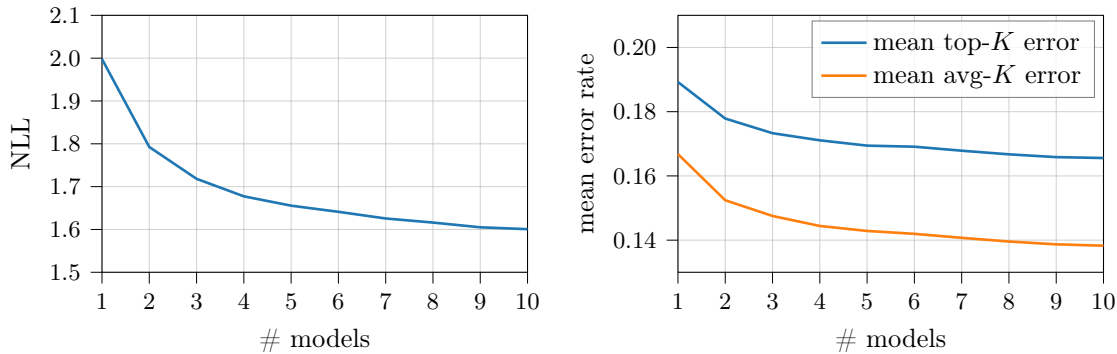}%
    \caption{\dataset{PlantCLEF2015}}
  \end{subfigure}

  \vspace{.5cm}
  
  \begin{subfigure}{\textwidth}
    \input{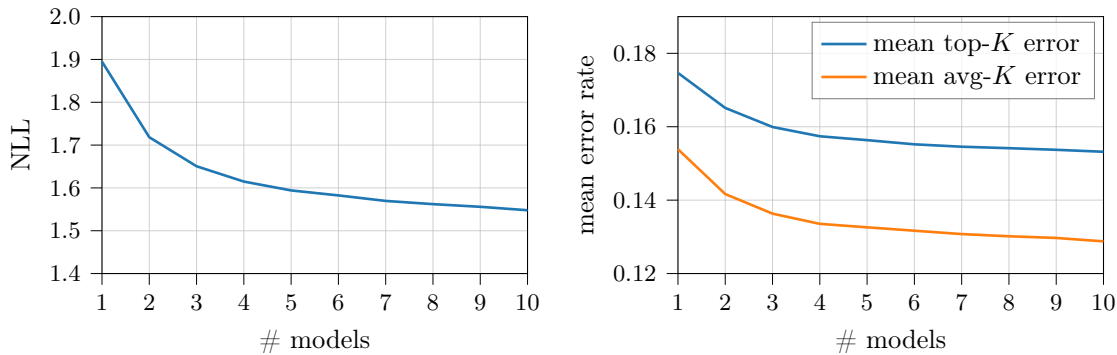}%
    \caption{\dataset{FGVC5-Fungi}}
  \end{subfigure}

  \vspace{.5cm}
  
  \begin{subfigure}{\textwidth}
    \input{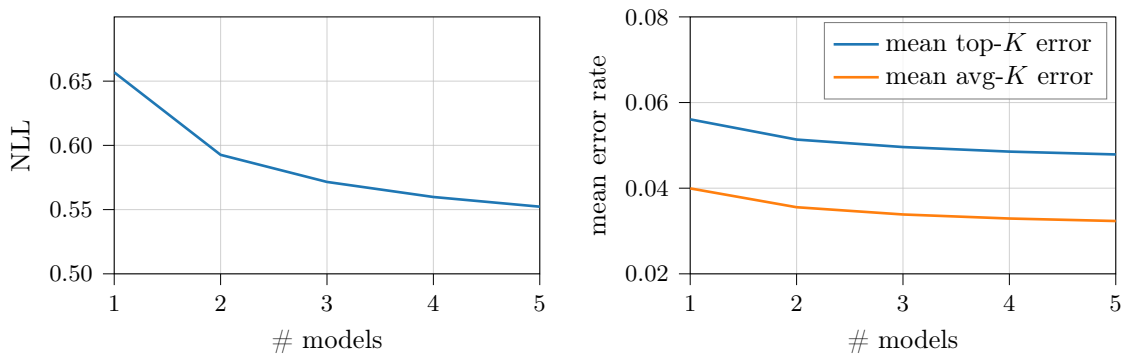}%
    \caption{\dataset{FGVC6-ButterfliesMoths}}
  \end{subfigure}

  \caption{
    Impact of the number of models present in the ensemble. 
    More models reduce both the negative log likelihood (NLL) and the error rate.
    The mean top-$K$ error is the averaged top-$K$ error rate when $K$ varies from 1 to 10.
    Similarly for mean average-$K$ error.
  }
  \label{fig:ensemble-impact-appendix}
\end{figure}

\newpage
\bibliography{bibliography.bib}

\end{document}